\documentclass[final,3p,times]{elsarticle}

\usepackage[textsize=scriptsize,textwidth=1cm]{todonotes}
\newcommand{\tm}[1]{\vspace{0.5em}\todo[author=Tim,inline,color=green]{#1}\vspace{0.5em}}

\newcommand{\gh}[1]{\vspace{0.5em}\todo[author=Guang,inline,color=red!30!white]{#1}\vspace{0.5em}}

\usepackage{amssymb}

\usepackage{lineno}
\usepackage{verbatim}
\usepackage{natbib}

\usepackage{amsmath}
\usepackage{amsthm}
\usepackage{booktabs}

\usepackage{xcolor}
\usepackage{multirow}
\usepackage{multicol}
\usepackage{url}
\usepackage{tikz}
\usetikzlibrary{calc,fadings,decorations.pathreplacing}
\usepackage{siunitx}

\newcommand{\showDefinition}{DEFINITION}
\newtheoremstyle{definitionstyle}
  {5pt}
  {5pt}
  {}
  {}
  {\bfseries}
  {:}
  {\newline}
  {}

\theoremstyle{definitionstyle}
\newtheorem{Definition}{\showDefinition}[subsection]

\newcommand{\showPro}{Proposition}
\newtheoremstyle{propositionstyle}
  {5pt}
  {5pt}
  {}
  {}
  {\itshape\bfseries}
  {:}
  {\newline}
  {}

\theoremstyle{propositionstyle}
\newtheorem{Proposition}{\showPro}[subsection]

\theoremstyle{definition}
\newtheorem{example}{Example}[section]
\theoremstyle{definition}
\newtheorem{theorem}{Theorem}[section]

\newcommand{\f}{\mathit{f}}
\newcommand{\cc}{\mathit{c}}
\newcommand{\extf}{\mathbb{F}}

\newcommand{\sense}{\textbf{sense}}
\newcommand{\move}{\textbf{move}}
\newcommand{\shout}{\textbf{shout}}
\newcommand{\checking}{\textbf{@check}}

\newcommand{\movetoleft}{\textbf{move\_left}}
\newcommand{\movetoright}{\textbf{move\_right}}
\newcommand{\share}{\textbf{share}}

\newcommand{\turn}{\textbf{turn}}

\newcommand{\post}{\textbf{post}}

\journal{Journal Name}

\begin{document}

\begin{frontmatter}



\title{What you get is what you see: Decomposing Epistemic Planning using Functional STRIPS}


\author{Guang Hu, Tim Miller, and Nir Lipovetzky}

\address{School of Computing and Information Systems\\The University of Melbourne, Australia\\ghu1@student.unimelb.edu.au, \{tmiller,nir.lipovetzky\}@unimelb.edu.au}

\begin{abstract}

Epistemic planning --- planning with knowledge and belief --- is essential in many multi-agent and human-agent interaction domains. 
Most state-of-the-art epistemic planners solve this problem by compiling to propositional classical planning, for example, generating all possible knowledge atoms, or compiling epistemic formula to normal forms.
However, these methods become computationally infeasible as problems grow. 
In this paper, we decompose epistemic planning by delegating reasoning about epistemic formula to an external solver. We do this by modelling the problem using \emph{functional STRIPS}, which is more expressive than standard STRIPS and supports the use of external, black-box functions within action models. 
Exploiting recent work that demonstrates the relationship between what an agent `sees' and what it knows, we allow modellers to provide new implementations of externals functions. These define what agents see in their environment, allowing new epistemic logics to be defined without changing the planner. As a result, it increases the capability and flexibility of the epistemic model itself, and avoids the exponential pre-compilation step.
We ran evaluations on well-known epistemic planning benchmarks to compare  with an existing state-of-the-art planner, and on new scenarios based on different external functions. The results show that our planner scales significantly better than the state-of-the-art planner against which we compared, and can express problems more succinctly.

\end{abstract}


\end{frontmatter}


\section{Introduction}

Automated planning is a model-based approach to study sequential decision problems in AI \cite{DBLP:series/synthesis/2013Geffner}. Planning models describe the environment and the agents with concise planning languages, such as STRIPS. It then submits the description of the model to a general problem solver in order to find a sequence of actions to achieve a certain desired goal state. 
The description of the problem, in general, tracks the changes in the state of the environment.
However, in many scenarios, an agent needs to reason about the knowledge or belief of other agents in the environment. This concept is known as \emph{epistemic planning} \citep{DBLP:journals/jancl/BolanderA11}, a research topic that brings together the knowledge reasoning and planning communities. 

Epistemic logic is a formal account to perform inferences and updates about an agent's own knowledge and belief, including group and common knowledge in the presence of multiple agents 
\cite{Hintikka1962-HINKAB}.
Epistemic planning is concerned about action theories that can reason not only about variables representing the state of the world, but also the belief and knowledge that other agents have about those variables. Thus, epistemic planning intends to find the best course of action taking into account practical performance considerations when reasoning about knowledge and beliefs \citep{DBLP:journals/jancl/BolanderA11}.  \citeauthor{DBLP:journals/jancl/BolanderA11} (\citeyear{DBLP:journals/jancl/BolanderA11}) first used event-based models to study epistemic planning in both single and multi agent environments, and gave a formal definition of epistemic planning problems using Dynamic Epistemic Logic (DEL) \citep{DBLP:journals/corr/Bolander17}. 

There are typically two frameworks in which epistemic planning are studied. First, is to use DEL. This line of research investigates the decidability and complexity of epistemic planning and studies what type of problems it can solve  \cite{DBLP:conf/ijcai/WanYFLX15,DBLP:conf/ijcai/HuangFW017,DBLP:conf/ksem/Wu18}. The second is to extend existing planning languages and solvers to epistemic tasks \cite{DBLP:conf/aaai/MuiseBFMMPS15,DBLP:conf/atal/MuiseMFPS15,DBLP:conf/aips/KominisG15,DBLP:conf/aips/KominisG17,DBLP:conf/aips/LeFSP18}. In this paper, we take the latter approach. 

The complexity of epistemic planning is undecidable in the general case. Thus, one of the main challenges of epistemic planning concerns computational efficiency. 
The dominant approach in this area it to rely on compilations. These solutions pre-compile epistemic planning problems   into classical planning problems, using off-the-shelf classical planners to find solutions \cite{DBLP:conf/aaai/MuiseBFMMPS15,DBLP:conf/atal/MuiseMFPS15,DBLP:conf/aips/KominisG15,DBLP:conf/aips/KominisG17}; or pre-compile the epistemic formulae into specific normal forms for better performance during search \cite{DBLP:conf/ijcai/HuangFW017,DBLP:conf/ksem/Wu18}. Such approaches have shown to be fast at planning, but the compilation is computationally expensive.

This paper departs from previous approaches in two significant ways. First, we propose a model that exploits recent insights defining  what an agent knows as a function of what it `sees' \cite{DBLP:conf/ecai/CooperHMMR16,DBLP:conf/atal/GasquetGS14}. \citet{DBLP:conf/ecai/CooperHMMR16} define `seeing relations' as modal operators that `see' whether a proposition is true, and then define knowledge of a proposition $p$ as $K_i p \leftrightarrow p \land S_i p$; that is, if $p$ is true and agent sees $p$, then it knows $p$. Thus, the seeing modal operator is equivalent to the `knowing whether' operator \cite{DBLP:journals/rsl/FanWD15,DBLP:conf/aaai/MillerFMPS16}. We generalise the notion of seeing relations to \emph{perspective functions}, which are functions that determine which variables an agent sees. The domain of variables can be discrete or continuous, not just propositional. The basic implementation of perspective functions is just the same as seeing relations, however, we show that by changing the definition of perspective functions, we can establish new epistemic logics tailored to specific domains, such as Big Brother Logic \cite{DBLP:conf/atal/GasquetGS14}: a logic about visibility and knowledge in two-dimensional Euclidean planes.

Second, we show how to integrate perspective functions within functional STRIPS models as \emph{external functions} \cite{DBLP:conf/ijcai/FrancesRLG17}. External functions are black-box functions implanted in a programming language (in our case, C++), that can be called within action models.  Epistemic reasoning is delegated to external functions, where epistemic formulae are evaluated lazily, avoiding the exponential blow-up from epistemic formulae present in other compilation-based approaches. This delegation effectively decomposes epistemic reasoning from search, and allows us to implement our approach in any functional STRIPS planner that supports external functions. Further, the modeller can implement new perspective functions that are tied to specific domains, effectively defining new external solvers.

In our experiments we use a width-based functional STRIPS planner \cite{DBLP:conf/ijcai/FrancesRLG17} that is able to evaluate the truth value of epistemic fluents with external solvers, and solve a wide range of epistemic problems efficiently, included but not limited to, nested knowledge, distributed knowledge and common knowledge. We also show how modellers can implement different perspective functions as external functions in the functional STRIPS language, enabling the use of  domain-dependent epistemic logics. Departing from propositional logic give us flexibility to encode expressive epistemic formulae concisely. We compare our approach to a state-of-the-art epistemic planner that relies on a compilation to classical planning \cite{DBLP:conf/aaai/MuiseBFMMPS15}. The results show that, unlike in the compilation based approaches, the depth of nesting and the number of agents does not affect our performance, avoiding the exponential blow up due to our lazy evaluation of epistemic formulae. 

In the following sections we give a brief background on both epistemic logic and epistemic planning (Section \ref{sec:background}). We then introduce a new model, the \textit{agent perspective model} (Section~\ref{sec:model}). We discuss implementation details using a functional STRIPS planner (Section~\ref{sec:implementation}), and report experiments on several well-known benchmarks, along with two new scenarios to demonstrate the expressiveness of the proposed approach (Section~\ref{sec:experiments}).

\section{Background}\label{sec:background}
In this section, we briefly introduce the three main areas related to this work: (1) classical planning; (2) epistemic logic; and (3) epistemic planning. 

\subsection{Classical Planning}
\label{sec:background:classical-planning}

    
    Planning is the model based approach to action selection in AI, where the model is used to reason about which action an agent should do next \cite{DBLP:series/synthesis/2013Geffner}. Models vary depending on the assumptions imposed on the dynamics of the world, from classical where all actions have deterministic instantaneous effects and the world is fully known, up to temporal or POMDP models, where actions have durations or belief distributions about the state of the world. Models are described concisely through declarative languages such as STRIPS and PDDL \cite{fikes:strips,mcdermott:pddl}, general enough to allow the encoding of different problems, while at the same time revealing important structural information that allow planners to scale up. In fact, most planners rely on exploiting the structure revealed in the action theory to guide the search of solutions, from the very first general problem solver \cite{simon:gps} up to the latest computational approaches based on SAT, and heuristic search \cite{rintanen2012planning,DBLP:journals/jair/RichterW10}. On the other hand, declarative languages have limited the scope of planning, as certain environments representing planning models are difficult to encode declaratively, but are easily defined through simulators such as the Atari video games \cite{bellemare2013}. Consequently, a new family of width-based planners \cite{DBLP:conf/ecai/LipovetzkyG12, DBLP:conf/ecai/LipovetzkyG14,DBLP:conf/aips/LipovetzkyG17} have been proposed, broadening the scope of planning and have been shown to scale-up even when the planning model is described through simulators, only requiring the exposure of the state variables, but not imposing any syntax restriction on the action theory \cite{DBLP:conf/ijcai/FrancesRLG17}, where the denotation of some symbols can be given procedurally through simulators or by external theories. 

    In this paper we focus on epistemic planning, which considers the \emph{classical planning model} as a tuple ${\cal S} = \langle S,s_0,S_G,Act,A,f,c \rangle$ where $S$ is a set of states, $s_0 \in S$ is the initial state, $S_G \subseteq S$ is the set of goal states, $Act$ is the set of actions,  $A(s)$ is the subset actions applicable in  state $s$, $f$ is the transition function so that $f(a,s)$  represents  the state $s'$ that results from doing action $a$ in the state $s$, and $c(a,s)$ is a cost function. The solution to a classical planning model $\cal S$, called a plan, is  a sequence of actions that maps the initial state into a goal state, i.e., $\pi=\langle a_0,\ldots, a_n \rangle$ is a plan if there is a sequence of states $s_0, \ldots, s_{n+1}$ such that $a_i \in A(s_i)$ and $s_{i+1}=f(a_i,s_i)$ for $i=0,\ldots,n$ such that $s_{n+1} \in S_G$. The cost of plan $\pi$ is given by the sum of action costs $c(a_i,s_i)$ and a plan is optimal if there is no plan with smaller cost.

    
    A classical planning model can be represented in STRIPS by a tuple $P= \langle F,O,I,G \rangle$ where:
    $F$ is the set of all possible facts or propositions, $O$ the set of all actions, $I\subseteq F$ a set of all true facts in the initial situation, and $G\subseteq F$ a set of facts that needs to be true as the goal conditions.
    
    
    Besides the model, a solver, which is also called a planner, plays another important role in planning. One of the most successful computational approaches to planning is heuristic search. Besides the choice of search algorithm, the key feature which distinguishes planners is the heuristic function chosen~\cite{DBLP:conf/aips/HelmertD09}. To achieve good performance, the heuristic funstions should be as informed as possible. For example, one of the current best-performing planner, LAMA, uses a landmark-based heuristic derived from the model~\cite{DBLP:journals/jair/RichterW10} along with other delete-relaxation heuristics \cite{DBLP:series/synthesis/2013Geffner}. The downside is that most heuristics require the model to be encoded in STRIPS or PDDL, but not all the problem can be easy to modeled declaratively. 
    
    
    The standard planning languages and solvers do not support the use of procedures or external theories. One exception is the FS planner \cite{frances2015modeling}, that supports the Functional STRIPS language \cite{Geffner2000}, where the denotation of (non-fluent) function symbols can be given extensionally by means of atoms, or intensionally by means of procedures. Procedures also appear as an  extension of PDDL under the name of semantic attachments \cite{dornhege2009semantic}. The reason why procedures are not   ``first-class citizens'' in planning languages is that there was  no clear  way to deal with them that is both general and effective. Recently, a new family of algorithms have been proposed in classical planning known as Width-based planning \cite{DBLP:conf/ecai/LipovetzkyG12}. The latest planner known as Best-First Width Search (BFWS) \cite{DBLP:conf/aaai/LipovetzkyG17}, has been shown to scale up even in the presence of functional symbols defined procedurally \cite{DBLP:conf/ijcai/FrancesRLG17}. $\mathit{BFWS}$  exploits a concept called \textit{novelty} in order to prune or guide the search. The novelty of a new state is determined by the presence of values in its state variables that are seen for the first time during the search. It keeps checking when generating search nodes whether there is something novel about them. The novelty of the node is the size of the minimal tuple of state variables that are seen for the fist time during the search.
    
    The f-STRIPS planner using $BWFS$ has been compared over 380 problems with respect of the performance of FF*~\cite{DBLP:journals/jair/Helmert06}, LAMA-11~\cite{DBLP:journals/jair/RichterW10}, and BFWS($f_5$)~\cite{DBLP:conf/aaai/LipovetzkyG17}. The results show that the f-STRIPS BFWS planner performances is as good as BFWS($f_5$), which relies on a STRIPS model, and slightly better than LAMA-11. It is worth to mention that FF and LAMA have been the top-performing planners for last 15 years, and BFWS($f_5$) has been shown to win the agile track in the last International Planning Competition 2018, and is a state-of-art planner for classical planning. The f-STRIPS BFWS planner thus can cope with externally defined functional symbols while performing well with respect to other planners. Therefore, we choose their planner as our developing tool.\footnote{ The planner is available through \url{https://github.com/aig-upf/2017-planning-with-simulators}.}
    
    \subsection{Epistemic Logic}
\label{sec:background:epistemic-logic}    
    In this section, we give the necessary preliminaries for epistemic logic -- the logic of knowledge. Knowledge in a multi-agent system is not only about the environment, but also about the agents' knowledge  about the environment, and of agents' knowledge of others' knowledge about the environment, and so on.
    
    \citeauthor{Fagin:2003:RK:995831} \cite{Fagin:2003:RK:995831} provides a formal definition of epistemic logic as follows. Given a set of countable set of all primitive propositions {$Prop = \{p_1,p_2,...\}$} and a finite set of agents {$Agt = \{a_1,a_2,...\}$}, the syntax for epistemic logic is defined as:
    \[
    \varphi ::= p \mid \varphi_1 \land \varphi_2 \mid \neg \varphi \mid K_i\varphi,
    \]
    in which $p \in Prop$ and $i \in Agt$.
    
     $K_i\varphi$ represents that agent $i$ knows proposition $\varphi$, $\neg$ means negation and $\land$ means conjunction. Other operators such as disjunction and implication can be defined in the usual way.
    
     \citeauthor{Fagin:2003:RK:995831} \cite{Fagin:2003:RK:995831}  define the semantics of epistemic logic using Kripke structures, as standard in modal logic. A Kripke structure is a tuple $M=(S,\pi,\mathcal{K}_1,\dots,\mathcal{K}_n)$ where:
    \begin{itemize}
        \item $S$ is a non-empty set of states;
        \item $\pi(s)$ is an interpretation function: $Prop\mapsto\{true \mid false\}$; and
        \item $\mathcal{K}_1\dots\mathcal{K}_n$ represents the \emph{accessibility relations} over states for each of the $n$ agents in $Agt$.
    \end{itemize}
    
    Given a state $s$ and a proposition $p$, the evaluation of $p$ over $s$ is $\pi(s)(p)$. If $p$ is true in $s$, then $\pi(s)(p)$ must be $true$, and vice-versa. $\mathcal{K}_i$ for agent $i$ is a binary relation over states, which is the key to reason about knowledge. For any pair of states $s$ and $t$, if $(s,t)\in \mathcal{K}_i$, then we can say agent $i$ cannot distinguish between $t$ and $s$ when in state $s$. In other words, if agent $i$ is in $s$, the agent can consider $t$ as the current state based on all the information it can obtain from state $s$. With this definition of Kripke structure, we can define the semantics of knowledge. 
    
    Given a state $s$, a proposition $p$, a propositional formula $\varphi$ and a Kripke structure $M$, the truth of two basic formulae are defined as follows:
    \begin{itemize}
        \item $(M,s)\ \vDash\ p$ iff $\pi(s)(p)$=$true$
        \item $(M,s)\ \vDash\ K_i \varphi$ iff $(M,t)\ \vDash\ \varphi$ for all such that $(s,t)\in \mathcal{K}_i$
    \end{itemize}
    $(M,s)\ \vDash\ p$ represents $p$ is true at $(M,s)$, which means, the evaluation of $p$ in state $s$, $\pi(s)(p)$ must  be true. Standard propositional logic rules define conjunction and negation. $(M,s)\ \vDash\ K_i \varphi$ is defined by agent $\varphi$ being true at all worlds $t$ reachable from $s$ via the accessibility relation $\mathcal{K}_i$.  This allows knowledge to be nested; for example, $K_a K_b p$ represents that agent $a$ knows that agent $b$ knows $p$, which means $p$ is true at all worlds reachable by applying accessibility relation $\mathcal{K}_a$ followed by $\mathcal{K}_b$. 
    
    From these basic operators, the concept of group knowledge can  be defined. For this, the grammar above is extended to:
    \[
    \varphi ::= p \mid \varphi_1 \land \varphi_2 \mid \neg \varphi \mid K_i\varphi \mid E_G\varphi \mid D_G\varphi \mid C_G\varphi,
    \]
    in which $p \in Prop$, $i \in Agt$, and $G$ is a non-empty set of agents such that $G \subseteq Agt$.
    
    $E_G\varphi$ represents that everyone in group $G$ knows $\varphi$ and $C_G\varphi$ represents that it is \emph{commonly known} in group $G$ that $\varphi$ is true, which means that everyone knows $\varphi$, and everyone knows that everyone knows $\varphi$, \emph{ad infinitum}. $D_G\varphi$ represents \emph{distributed} knowledge, which means we combine the knowledge of the set of agents $G$ such that $G$ knows $\varphi$, even though it may be that no individual in the group knows $\varphi$. 
    
    The semantics for these group operators are defined as follows:
    \begin{itemize}
        \item $(M,s)\ \vDash\ E_G \varphi$ iff $(M,s)\ \vDash\ K_i \varphi$ for all $i \in G$; 
        \item $(M,s)\ \vDash\ C_G \varphi$ iff $(M,t)\ \vDash\ \varphi$ for all $t$ that are $G$-reachable from $s$
        \item $(M,s)\ \vDash\ D_G \varphi$ iff $(M,t)\ \vDash\ \varphi$ for all $t$ such that $(s,t) \in \cap_{i \in G} \mathcal{K}_i$
    \end{itemize}
    
     By definition, $(M,s)\ \vDash\ E_G \varphi$ will be true, if and only if, $\varphi$ is known by all agents in $G$. Common knowledge in world $s$, $(M,s)\ \vDash\ C_G \varphi$ is defined as: in all worlds $t$, which are reachable by following the accessibility relations of all agents in $G$, $\varphi$ is true. For distributed knowledge, $(M,s)\ \vDash\ D_G \varphi$ is true, if and only if, in all the possible worlds that any one agent from $G$ consider possible, $\varphi$ is true; in other words, we eliminate worlds that any agent in $G$ knows to be impossible.
    
    \subsubsection{Seeing and Knowledge}
    
    Recently \citeauthor{DBLP:conf/atal/GasquetGS14}\cite{DBLP:conf/atal/GasquetGS14} noted the relationship between what an agent sees and what it knows. They define a more specific task of logically modeling and reasoning about cooperating tasks of vision-based agents, which they called \emph{Big Brother Logic} (BBL).  Their framework models multi-agent knowledge in a continuous environment of vision, which has great potential applications such as reasoning over cameras inputs, autonomous robots and vehicles. They introduced the semantics of their model and its extensions on natural geometric models.
    
    \begin{figure}
        \centering

\newcommand{\range}{30}
\newcommand{\size}{3}
\newcommand{\oversize}{6}
\makeatletter
\newcommand{\gettikzxy}[3]{%
  \tikz@scan@one@point\pgfutil@firstofone#1\relax
  \edef#2{\the\pgf@x}%
  \edef#3{\the\pgf@y}%
}
\makeatother

\begin{tikzpicture}

    \begin{scope}[transparency group]
        \begin{scope}[blend mode=multiply]
            \fill[ opacity=0.5,blue!30] (0,0) -- ({cos(\range)*\oversize},{sin(\range)*\oversize}) -- ({cos(\range)*\size*4},{sin(\range)*\oversize}) -- ({cos(\range)*\size*4},-{sin(\range)*\oversize}) --({cos(\range)*\oversize},-{sin(\range)*\oversize}) -- cycle;
            \fill[ opacity=0.5,yellow!30] ({cos(\range)*\size*2},0) -- ({cos(\range)*\size*2-cos(\range)*\oversize},{sin(\range)*\oversize}) -- (-{cos(\range)*\size*2},{sin(\range)*\oversize}) -- (-{cos(\range)*\size*2},-{sin(\range)*\oversize}) -- ({cos(\range)*\size*2-cos(\range)*\oversize},-{sin(\range)*\oversize}) -- cycle;
        \end{scope}
    \end{scope}

    \coordinate (origin) at (0,0);
    \coordinate (a1) at (0,0);
    \coordinate (a1up) at ({cos(\range)*\oversize},{sin(\range)*\oversize});
    \coordinate (a1down) at ({cos(\range)*\oversize},-{sin(\range)*\oversize});
    \coordinate (a2) at ({cos(\range)*\size*2},0);
    \coordinate (a2up) at ({cos(\range)*\size*2-cos(\range)*\oversize},{sin(\range)*\oversize});
    \coordinate (a2down) at ({cos(\range)*\size*2-cos(\range)*\oversize},-{sin(\range)*\oversize});

    \draw[thick] let    \p{1} = (a1)    in (a1 |- origin)
        node[circle,fill=red!80,inner sep=2pt,
        label={[align=center]below:
                $a_1$ \\ 
                (\pgfmathparse{(\x1/28.45274)}\num[round-mode=places,round-precision=1]{\pgfmathresult},
                \pgfmathparse{(\y1/28.45274)}\num[round-mode=places,round-precision=1]{\pgfmathresult})
                }] at (\x1,\y1) {};
    \draw[->, gray, thick] (a1) -- (a1up);
    \draw[->, gray, thick] (a1) -- (a1down);

    \draw[thick] let    \p{1} = (a2)    in (a1 |- origin)
        node[circle,fill=red!80,inner sep=2pt,
        label={[align=center]below:
                $a_2$ \\ 
                (\pgfmathparse{(\x1/28.45274)}\num[round-mode=places,round-precision=1]{\pgfmathresult},
                \pgfmathparse{(\y1/28.45274)}\num[round-mode=places,round-precision=1]{\pgfmathresult})
                }] at (\x1,\y1) {};
    \draw[->, gray, thick] (a2) -- (a2up);
    \draw[->, gray, thick] (a2) -- (a2down);
    
    \coordinate (top_left) at (-{cos(\range)*\size*2},{sin(\range)*\oversize});
    \coordinate (bottom_left) at (-{cos(\range)*\size*2},-{sin(\range)*\oversize});
    \coordinate (top_right) at ({cos(\range)*\size*4},{sin(\range)*\oversize});
    \coordinate (bottom_right) at ({cos(\range)*\size*4},-{sin(\range)*\oversize});
    
    \draw[black,thick] (top_left) -- (bottom_left) -- (bottom_right) -- (top_right) -- cycle;


    \coordinate (b1) at (-{cos(\range)*\size},0);
    \coordinate (b2) at ({cos(\range)*\size},0);
    \coordinate (b3) at ({cos(\range)*\size*3},0);
    \coordinate (b4) at ({cos(\range)*\size},{sin(\range)*\oversize*3/4});
    
    \draw[thick] let    \p{1} = (b1)    in (a1 |- origin)
        node[circle,fill=black!80,inner sep=1pt,
        label={[align=center]below:
                $b_1$ \\ 
                (\pgfmathparse{(\x1/28.45274)}\num[round-mode=places,round-precision=1]{\pgfmathresult},
                \pgfmathparse{(\y1/28.45274)}\num[round-mode=places,round-precision=1]{\pgfmathresult})
                }] at (\x1,\y1) {};
                
    \draw[thick] let    \p{1} = (b2)    in (a1 |- origin)
        node[circle,fill=black!80,inner sep=1pt,
        label={[align=center]below:
                $b_2$ \\ 
                (\pgfmathparse{(\x1/28.45274)}\num[round-mode=places,round-precision=1]{\pgfmathresult},
                \pgfmathparse{(\y1/28.45274)}\num[round-mode=places,round-precision=1]{\pgfmathresult})
                }] at (\x1,\y1) {};    

    \draw[thick] let    \p{1} = (b3)    in (a1 |- origin)
        node[circle,fill=black!80,inner sep=1pt,
        label={[align=center]below:
                $b_3$ \\ 
                (\pgfmathparse{(\x1/28.45274)}\num[round-mode=places,round-precision=1]{\pgfmathresult},
                \pgfmathparse{(\y1/28.45274)}\num[round-mode=places,round-precision=1]{\pgfmathresult})
                }] at (\x1,\y1) {};    

    \draw[thick] let    \p{1} = (b4)    in (a1 |- origin)
        node[circle,fill=black!80,inner sep=1pt,
        label={[align=center]below:
                $b_4$ \\ 
                (\pgfmathparse{(\x1/28.45274)}\num[round-mode=places,round-precision=1]{\pgfmathresult},
                \pgfmathparse{(\y1/28.45274)}\num[round-mode=places,round-precision=1]{\pgfmathresult})
                }] at (\x1,\y1) {};
                

\end{tikzpicture}



        \caption{Example for big brother logic}
        \label{fig:big_brother_logic}
    \end{figure}
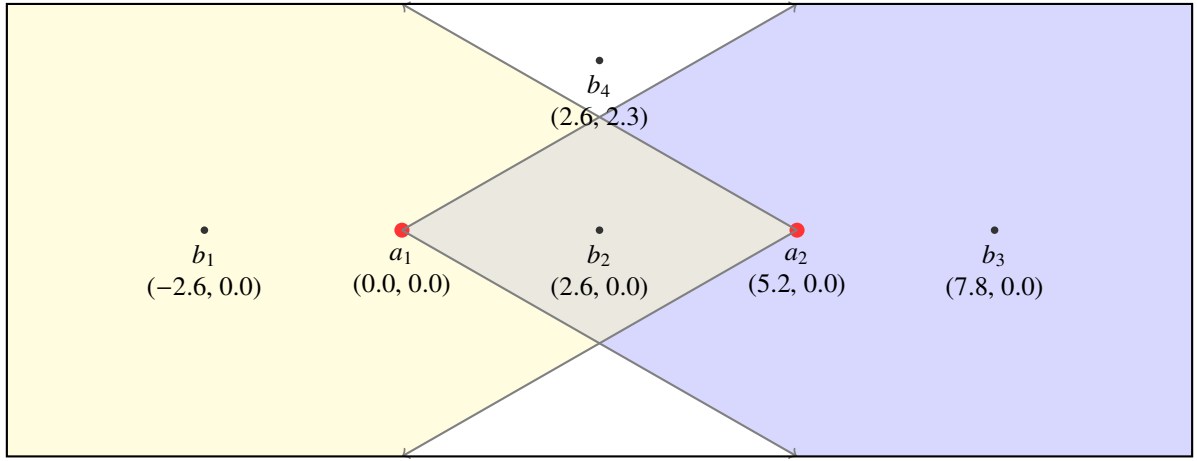

    In their scenario, agents are stationary cameras in a Euclidean plane $\mathbb{R}^2$, and they set the assumptions that those cameras can see anything in their sight range, and they do not have volume, which means they would not block others' sight. They extend \citeauthor{Fagin:2003:RK:995831}'s logic \cite{Fagin:2003:RK:995831} by noting that, at any point in time, what an agent knows, including nested knowledge, can be derived directly from what it can see in the current state. In brief, instead of Kripke frames, they define a \emph{geometric model} as $(pos,dir,ang)$, in which:
    
    \begin{itemize}
        \item $pos : Agt \rightarrow \mathbb{R}^2$
        \item $dir : Agt \rightarrow U$
        \item $ang : Agt \rightarrow [0, 2\pi)$
    \end{itemize}
    in which $U$ is the set of unit vectors of $\mathbb{R}^2$, the $pos$ function gives the position of each agent, the $dir$ function gives the direction that each agent is facing, and the function $ang$ gives the angle of view for each agent.

    A model is defined as $(pos, ang, D, R)$, in which $pos$ and $ang$ are as above, $D$ is the set of $dir$ functions and $R$ the set of equivlance relations, one for each agent $a$, defined as: 
    \[
     R_a = \{(dir, dir') \in D^2 \mid \textrm{ for all } b \neq a, dir(b) = dir'(b)\}
     \]
    In this context, standard propositional logic is extended with the binary operator $a \triangleright b$, which represents that ``$a$ sees $b$". This is defined as:
    \begin{itemize}
        \item $(pos,ang,D,R), dir \vDash a \triangleright b$ iff $pos(b) \in C_{pos(a),dir(a),ang(a)}$,
    \end{itemize}  
    in which $C_{pos(a),dir(a),ang(a)}$ is the field of vision that begins at $pos(a)$ from direction $dir(a)$ and goes $ang(a)$ degrees in a counter-clockwise direction.
    
    Figure~\ref{fig:big_brother_logic} shows an example with two agents, $a_1$ and $a_2$, and  model $((0.0,0.0),60^\circ,D,R)$ and $((5.2,0.0),60^\circ,D,R)$ respectively, along with four objects, $b_1$, $b_2$, $b_3$ and $b_4$. Based on the current state, for agent $a_1$, we have: 
    \newline $(pos,ang,D,R),dir \vDash a_1 \triangleright a_2$; $(pos,ang,D,R),dir \vDash a_1 \triangleright b_2$; and, $(pos,ang,D,R),dir \vDash a_1 \triangleright b_3$.
    
    From this, they show the relationship between seeing and knowing. For example, $K_a (b \triangleright c)$ is defined as $a \triangleright b \land  a \triangleright c \land b \triangleright c$. 
    
    \citeauthor{DBLP:conf/atal/GasquetGS14} \cite{DBLP:conf/atal/GasquetGS14} also define a common knowledge operator, defined in a similar manner to that of \citeauthor{Fagin:2003:RK:995831}'s definition based on $G$-reachable worlds \cite{Fagin:2003:RK:995831}. In Figure~\ref{fig:big_brother_logic}, the formula $C_{\{a_1,a_2\}loc(b_2)=(2.6,0.0)}$ holds because $a_1$ and $a_2$ can both see $b_2$, and can both see that each other sees $b_2$, etc.

    \citeauthor{DBLP:conf/ecai/CooperHMMR16}~\cite{DBLP:conf/ecai/CooperHMMR16} adopted \citeauthor{DBLP:conf/atal/GasquetGS14}~\cite{DBLP:conf/atal/GasquetGS14}'s idea of modelling an agent's knowledge based on what it sees, and generalised it to seeing propositions, rather than to just to seeing other agents in a Euclidean plane. They extended \citeauthor{Fagin:2003:RK:995831}'s definition by adding an extra formula $\alpha$, which represents formula about what can be seen:
    \[
    \begin{array}{lll}
    	 \alpha & ::= & p \mid S_i p \mid S_i\alpha\\
    	\varphi & ::= & \alpha \mid \varphi_1 \land \varphi_2 \mid \neg \varphi \mid K_i\varphi, 
    \end{array}
    \]
    in which $p \in PVar$ (the set of propositional \emph{variables}) and $i \in Agt$. The grammar for $\alpha$ defines how to represent visibility relations. $S_i\alpha$ can be read as ``agent $i$ sees $\alpha$". Note the syntactic restriction that agents can only see atomic propositions or nestings of seeing relationships that see atomic propositions. 
    
    From this, they define knowledge using the equivalences $K_i p \leftrightarrow p \land S_i p$ and $K_i \neg p \leftrightarrow \neg p \land S_i p$. This tight correspondence between seeing and knowledge is intuitive: an agent knows $p$ is true if $p$ is true and the agent can see the variable $p$. Such a relationship is the same as the relationship between knowing something is true and \emph{knowing whether} something is true \cite{DBLP:conf/aaai/MillerFMPS16,DBLP:journals/rsl/FanWD15}. In fact, in early drafts of the work available online, \citeauthor{DBLP:conf/ecai/CooperHMMR16}~\cite{DBLP:conf/ecai/CooperHMMR16}, $S_i$ was written $KW_i$ and was called the ``knowing whether'' operator.
    
    
    Comparing these two bodies of work,  \citet{DBLP:conf/atal/GasquetGS14} use a geometric model to represent the environment and derive knowledge from this by checking the agents' line of sight. Their idea is literally matching the phrase ``Seeing is believing". However, their logic is constrained only to vision in physical spaces. While in \citet{DBLP:conf/ecai/CooperHMMR16}'s world, the seeing operator applies to propositional variables, and thus visibility can be interpreted more abstractly; for example, `seeing' (hearing) a message over a telephone.
    
    The current paper generalises seeing relations to perspective functions, which are domain-dependent functions defining what agents see in particular states. The result is more flexible than seeing relations, and allows Big Brother Logic to be defined with a simple perspective function, as well as new perspective functions for other domains; for example, Big Brother Logic in three-dimensional planes, or visibility of messages on a social network.
    
    \subsection{Epistemic Planning}		
    
    \citet{DBLP:journals/jancl/BolanderA11} introduce the concept of epistemic planning, and define it in both single agent and multi-agent domain. Their planning framework is defined in \emph{dynamic epistemic logic} (DEL) \cite{van2007dynamic}, which has been shown to be undecidable in general, but with decidable fragements~\cite{DBLP:conf/ijcai/BolanderJS15}. In addition, they also provided  a solution to PSPACE-hardness of the plan verification problem. This formalism has been used to explore theoretical properties of epistemic planning; for example, \citeauthor{DBLP:journals/corr/EngesserBMN17}~\cite{DBLP:journals/corr/EngesserBMN17} used concepts of perspective shifting  to reason about other's contribution to  joint goals. Along with implicit coordination actions, their model can solve some problems elegantly without communication between agents.
    
    
    Since epistemic planning was formalized in DEL, there has been substantial work on DEL-based planning. However, in this paper, our focus is on the design, implementation, and evaluation of planning tools, rather than on logic-based models of planning. In this section, we focus on research in epistemic planning that is focused on the design, implementation, and evaluation of planning algorithms.
    
    A handful of researchers in the planning field focus on leveraging existing planners to solve epistemic problems.  \citeauthor{DBLP:conf/aaai/MuiseBFMMPS15}~\cite{DBLP:conf/aaai/MuiseBFMMPS15} proposed an approach to multi-agent epistemic planning  with nested belief, non-homogeneous agents, co-present observation, and the ability for one agent to reason as if it were the other. Generally, compiled an epistemic logic problem into a classical planning problem by grounding epistemic fluents into propositional fluents and using additional conditional effects of actions to enforce  desirable properties of belief.  They evaluated their approach on the \textsl{Corridor}~\cite{DBLP:conf/aips/KominisG15} problem and the \textsl{Grapevine} problem, which is a combination of \textsl{Corridor} and \textsl{Gossip}~\cite{DBLP:conf/eumas/HerzigM15}. Their results show that their approach is able to solve the planning task within a typically short time, but the compilation time to generate fluents and conditional effects is exponential in the size of the original problem. 
    
    Simultaneously, \citeauthor{DBLP:conf/aips/KominisG15}~\cite{DBLP:conf/aips/KominisG15} adopted methods from partially-observable planning for representing beliefs of a single agent, and converted that method to handle multi-agent setting. They define three kinds of actions sets $O$, $N$ and $U$. $O$ represents all physical actions, which is the same $O$ as in classical planning. The actions set $N$ denotes a set of sensing actions, which can be used to infer knowledge. The last action set $U$ is used to update any of the fact $\varphi$. They are able to use their model to encode epistemic planning problems, and can convert their model to a classical planning model using standard compilation techniques for partially-observable planning. They evaluated their model on \textsl{Muddy Children}, \textsl{Sum} and \textsl{Word Rooms}~\cite{DBLP:conf/aips/KominisG15} domains. Their results show that their model is able to handle solve all cases presented, however, the performance as measured by time and depth of the optimal plan is indirectly related to the problem scope and planning algorithm. For example, by using two different planners, $A^*(h_{max})$ and BFS($h_{add}$), the \textsl{Muddy Children} test case with seven children, the time consumed by BFS($h_{add}$) is only half of $A^*(h_{max})$, while for other problems, there is little difference between the two.
    
    Since \citeauthor{DBLP:conf/aips/KominisG15}, and \citeauthor{DBLP:conf/aaai/MuiseBFMMPS15}'s approach the problem in a similar way (compilation to classical planning) their results are similar in a general way. However, the methods they used are different, so, their work and results have diverse limitations and strengths. For \citeauthor{DBLP:conf/aaai/MuiseBFMMPS15}'s work, they managed to modeling nested beliefs without explicit or implicit Kripke structures, which means they can only represent literals, while \citeauthor{DBLP:conf/aips/KominisG15}'s work is able to handle arbitrary formulae. Furthermore, \citeauthor{DBLP:conf/aaai/MuiseBFMMPS15}'s model does not have the strict common initial knowledge setting found in \citeauthor{DBLP:conf/aips/KominisG15}, and does not have the constraint that all action effects are commonly known to all the agents. Therefore, \citeauthor{DBLP:conf/aaai/MuiseBFMMPS15}'s model allows them to model beliefs, which might be not what is actual true in the world state, rather just model knowledge. In other words, they can handle different agents having different belief about the same fluent. 
    
    More recently, rather than compiling epistemic planning problems into classic planning, \citeauthor{DBLP:conf/ijcai/HuangFW017} built a native multi-agent epistemic planner, and proposed a general representation framework for multi-agent epistemic problems~\cite{DBLP:conf/ijcai/HuangFW017}. They considered the whole multi-agent epistemic planning task from a third person point of view. In addition, based on a well-established concepts of believe change algorithms (both revision and update algorithms), they design and implemented an algorithm to encode belief change as the result of planning actions. Catering to their ideas and algorithms, they developed a planner called MEPK. They evaluated their approach with \textsl{Grapevine}, \textsl{Hexa Game} and \textsl{Gossip}, among others. From their results, it is clear that their approach can handle a variety of problems, and performance on some problems is better than other approaches. While this approach is different  from \citeauthor{DBLP:conf/aips/KominisG15} and \citeauthor{DBLP:conf/aaai/MuiseBFMMPS15}, it still requires a compilation phase before planning to re-write epistemic formula into a specific normal form called \emph{alternating cover disjunctive formulas} (ACDF) \cite{hales2012refinement}. The ACDF formula is worst-case exponentially longer than the original formula. The results show that this step is of similar computational burden as either \citeauthor{DBLP:conf/aips/KominisG15} or \citeauthor{DBLP:conf/aaai/MuiseBFMMPS15}. In addition, building a native planner to solve an epistemic planning problem makes it more difficult to take advantage of recent advances in other areas of planning.

\section{A Model of Epistemic Planning using Perspectives}
\label{sec:model}
In this section, we define a syntax and semantics of our \emph{agent perspective model}. Our idea is based on that of Big Brother Logic \cite{DBLP:conf/atal/GasquetGS14} and \citeauthor{DBLP:conf/ecai/CooperHMMR16}'s seeing operators \cite{DBLP:conf/ecai/CooperHMMR16}. The syntax and semantics of our model is introduced, including distributed knowledge and common knowledge. 

\subsection{Language}


Extending \citeauthor{DBLP:conf/ecai/CooperHMMR16}'s \cite{DBLP:conf/ecai/CooperHMMR16} idea of seeing propositional variables, our model is based on a model of functional STRIPS (F-STRIPS) \cite{DBLP:conf/ijcai/FrancesRLG17}, which uses variables and domains, rather than just propositions. We allow agents to see variables with discrete and continuous domains. 

\subsubsection{Model}\label{models}

We define an epistemic planning problem as a tuple $(Agt,V,D,O,I,G)$, in which $Agt$ is a set of agents, $V$ is a set of variables, $D$ stands for domains for each variable, in which domains can be discrete or continuous, $I$ and $G$ are the initial state and goal states respectively, and both of them are also bounded by $V$ and $D$. Specifically, they should be assignments for some or all variables in $V$, or relations over these. $O$ is the set of operators, with argument in the terms of variables from $V$. 

\subsubsection{Epistemic Formulae}
\label{grammar}
\begin{Definition}
Goals, actions preconditions, and conditions on conditional effects, are epistemic formulae, defined by the  following grammar:
\[
    \varphi ::= R(v_1,\dots,v_k) \mid \neg \varphi \mid \varphi \land \varphi \mid \neg \varphi \mid \varphi_1 \land \varphi_2 \mid S_i v \mid S_i \varphi \mid K_i \varphi,
\]
in which $R$ is $k$-arity `domain-dependent' relation formula, $v \in V$,  $S_i v$  and $S_i \varphi$ are both are visibility formulae, and $K_i \varphi$ is a knowledge formula.
\end{Definition}

\subsubsection{Domain Dependent Formulae}\label{domain_dependent_forumla}
Domain-dependent formula not only including basic mathematical relations, but also relational terms defined by the underlying planning language can have the relation between variables. For example, based on the scenario from Figure \ref{fig:big_brother_logic}, $loc(a_1)=(0,0)$ is a true formula expression the location of agent $a_1$, while $loc(a_1)=loc(a)$ is a false. In Section~\ref{sec:implementation}, we discuss the use of \emph{external functions} in F-STRIPS in our implemented planning, which allow
more complex relations, or even customized relations, as long as they have been defined in the external functions. For example, we can define a domain dependent relation in external function to compare distance between objects, called $@\mathit{far}\_\mathit{away}(loc(i),loc(j),loc(k))$. This external function takes three coordinates as input, and returns a Boolean value, whether distance between $i$ and $j$ is longer than $i$ and $k$. In the scenario displayed in Figure \ref{fig:big_brother_logic}, the relation $@\mathit{far}\_\mathit{away}(loc(a_1),loc(b_4),loc(b_2))$ would be true, while relation  $@\mathit{far}\_\mathit{away}(loc(a_1),loc(b_1),loc(b_2))$ would be false, since $b_1$ and $b_2$ are equally close to $a_1$. The definition of this function is delegated to a function implemented in a programming language such as C++, and the planner is unaware of its semantics. However, for the remainder of this section, we will ignore the existence of external functions, and return to them in our implementation section. 

\subsubsection{Visibility Formula}
An important concept adapted from \citet{DBLP:conf/ecai/CooperHMMR16} is ``seeing a proposition". Let $p$ be a proposition, ``agent $i$ knows whether $p$" can be represented as ``agent $i$ sees $p$". Their interpretation on this is: either $p$ is true and $i$ knows that; or, $p$ is false and $i$ knows that. With higher-order observation added, it gives us a way to reason about others' epistemic viewpoint about a proposition without actually knowing whether it is true. Building on this concept, our `seeing' operator allows us to write formulae about visibility: $S_i v$ and, $ S_i \varphi $.


The seeing formulae represent two related interpretations: seeing a variable; or seeing a formula. The formula $S_i v$ can be understood as variable $v$ has some value, and no matter what value it has, agent $i$ can see the variable and knows its value. On the other hand, seeing a relation is trickier. $S_v \varphi$ can be interpreted as: one relation $\psi $ is a formula and no matter whether  is true or false, $i$ knows whether it is true or not. To make sure $i$ knows whether $\varphi $ is true or not, the evaluation for this seeing formula is simply that agent $i$ sees the variables in that relation.

For example, in Figure \ref{fig:big_brother_logic}, using the notation defined in Section~\ref{models}, $S_{a_1} loc(b_1)$ can be read as ``agent $a_1$ sees variable $b_1$". In the case of seeing a domain-dependent formula, $S_{a_1} (\mathit{far}\_\mathit{away}, loc(a_1), loc(b_4), loc(b_2))$ can be read as ``agent $a_1$ sees whether the relation $\mathit{far}\_\mathit{away}(loc(a_1),loc(b_4),loc(b_2))$ is true or not", which is: ``agent $a_1$ sees whether $b_4$ is farther away from $a_1$ than $b_2$."

\subsubsection{Knowledge Formula}
In addition to the visibility operator, our language supports the standard knowing operator $K_i$.
Following novel idea from \citeauthor{DBLP:conf/ecai/CooperHMMR16}\cite{DBLP:conf/ecai/CooperHMMR16} on defining knowledge based on visibility \cite{DBLP:conf/ecai/CooperHMMR16}, we define knowledge as: $K_i\varphi \leftrightarrow \varphi \land S_i\varphi$. That is, for $i$ to know $\varphi$ is true, it needs to be able to see $\varphi$, and $\varphi$ needs to be true. In other words, if you can see something and it is true, then, you know it is true.

\subsection{Semantics}
\label{semantics}
We now formally define the semantics of our model.

\subsubsection{Knowledge Model}\label{subsubsec:knowledge-model}

Our model decomposes the planning model from the knowledge model, and as we will see in our implementation, our planner delegates the knowledge model to an external solver. Therefore, in this section, we define the semantics of that knowledge solver. The novel part of this model is the use of perspective functions, which are functions that define which variables an agent can see, instead of Kripke structures. From this, a rich knowledge model can be built up independent of the planning domain.

\begin{Definition}
A model $M$ is defined as $M=(V, D,\pi,\f_1,\dots,\f_n)$, in which $V$ is a set of variables, $D$ is a function mapping domains for each variable, $\pi$ is the evaluation function, which evaluates variables and formulae based on the given state, and $\f_1,\dots,\f_n : S \rightarrow S$ are the agents' \emph{perspective functions} that given a state $s$, will return the local state from agents' perspectives. 

A state $s$ is a tuple of variable assignments, denoted $[v_1=e_1, \dots, v_k=e_k]$, in which $v_i\in V$ and $e_i\in D(V)$ for each $i$. The \emph{global state} is a complete assignment for all variables in $V$. A \emph{local state}, which represents an individual agent's perspective of the global state, can be a partial assignment. The set of all states is denoted $S$.
\end{Definition}



A perspective function, $f_i : S \rightarrow S$ is a function that takes a state and returns a subset of that state, which represents the part of that state that is visible to agent $i$. These functions can be nested, such that $f_j(f_i(s))$ represents agent $i$'s perspective of agent $j$'s perspective, which can be just a subset of agent $j$'s actual perspective. The following properties must hold on $f_1,\ldots,\f_n$ for all $i \in Agt$ and $s \in S$:
\begin{enumerate}
 \item $f_i(s) \subseteq s$ 
 \item $f_i(s) = f_i(f_i(s))$
\end{enumerate}

First, we give the definition for propositional formulae. Let any variable be denoted as $v_i$, and $R$ any $k$-ary domain dependent formula:
\begin{itemize}
    \item $(M,s)\ \vDash\ R(v_1,\dots,v_k)$ iff $\pi(s)(R(\pi(s)(v_1),\dots,\pi(s)(v_k)))=true$
\end{itemize}

Relations are handled by evaluation function $\pi(s)$. The relation $R$ is evaluated by getting the value for each variable in $s$, and checking whether $R$ stands or not. Other propositional operators are defined in the standard way. 

Then, we  have the following formal semantics for visibility:
\begin{itemize}
    \item $(M,s)\ \vDash\ S_i\ v$ iff $\exists (v=x) \in \f_i(s)$ for any value $x$
    \item $(M,s)\ \vDash\ S_i\ R(v_1,\dots,v_k)$ iff $\forall v \in \{v_1,\dots,v_k\},\ (M,s)\ \vDash\ S_i v$
    \item $(M,s)\ \vDash\ S_i \neg \varphi$ iff $(M,s)\ \vDash S_i\ \varphi$
    \item $(M,s)\ \vDash\ S_i\ (\varphi\land\psi)$ iff $(M,s)\ \vDash S_i\ \varphi$ and $(M,s)\ \vDash S_i\ \psi$
    \item $(M,s)\ \vDash\ S_i\ S_j\ \varphi$ iff $(M,\f_i(s))\ \vDash\ S_j\ \varphi$
    \item $(M,s)\ \vDash\ S_i\ K_j\ \varphi$ iff $(M,\f_i(s))\ \vDash\ K_j\ \varphi$
\end{itemize}

$S_i v$, read ``Agent $i$ sees variable $v$", is true if and only if $v$ is visible in the state $\f_i(s)$. That is, an agent sees a variable if and only if that variable is in its perspective of the state. Similarly, an agent knows whether a domain-dependent formula is true or false if and only if it can see every variable of that formula. For example, in Figure~\ref{fig:big_brother_logic}, $S_{a_1} b_1$ is false and $S_{a_1} b_2$ is true, which is because $b_2$ is in $a_1$'s perspective (blue area), while $b_1$ is not. The remainder of the definitions simply deal with logical operations. However, note the definition of $S_i\neg\phi$, which is equivalent to $S_i \phi$, because all variables are the same. This effectively just defines that `seeing' a formula means seeing its variables.

Now, we define knowledge as:
\begin{itemize}
    \item $(M,s)\ \vDash\ K_{i}\ \varphi$ iff $(M,s)\ \vDash\ \varphi$ and $(M,s)\ \vDash\ S_i \varphi$
\end{itemize}

This definition follows  \citeauthor{DBLP:conf/ecai/CooperHMMR16}'s definition \cite{DBLP:conf/ecai/CooperHMMR16}: agent $i$ knows $\varphi$  if and only if the formula is true at $(M,s)$ and agent $i$ sees it. Using the same example as previously, $K_{a_1}@\mathit{far}\_\mathit{away}(loc(a_1),loc(b_2),loc(b_3))$ is false, while $K_{a_1} @\mathit{far}\_\mathit{away}(loc(a_1),loc(b_3),loc(b_2))$ is true.


\begin{theorem}
    The S5 axioms of epistemic logic hold in this language. That is, the following axioms hold: 
    
    \begin{tabular}{llll}
    (K) & $K_i (\varphi \rightarrow \psi)$ & $\rightarrow$ & $K_i \varphi \rightarrow K_i \psi$\\
    (T) & $K_i \varphi$ & $\rightarrow$ & $\varphi$\\
    (4) & $K_i \varphi$ & $\rightarrow$ & $K_i K_i \varphi$\\
    (5) & $\neg K_i \varphi$ & $\rightarrow$ & $K_i \neg K_i \varphi$
    \end{tabular}
\end{theorem}

\begin{proof}
    We first consider axiom (T). By our semantics, $(M,s)\ \vDash\ K_{i}\ \varphi$ is true, if and only if, both $(M,s)\ \vDash\ \varphi$ and $(M,s)\ \vDash\ S_i \varphi$ are true. Therefore,  it is trivial that axiom (T) holds. 
    
    For (K), based on our definition of knowledge, we have $(M,s)\ \vDash\ K_i (\varphi \rightarrow \psi)$ is equivalent to $(M,s)\ \vDash\ S_i (\varphi \rightarrow \psi)$ and $(M,s)\ \vDash\ (\varphi \rightarrow \psi)$. Then, by our semantics, we have that $(M,s)\ \vDash\ S_i (\varphi \rightarrow \psi)$ is equivalent to $(M,s)\ \vDash\ S_i \neg \varphi$ or $(M,s)\ \vDash\ S_i \psi$. From propositional logic, $\varphi \implies \psi$ is equivalent to $\neg\varphi \lor \psi$. We combine $(M,s)\ \vDash\ \neg \varphi$ with $(M,s)\ \vDash\ S_i \varphi$ to get $(M,s)\ \vDash \neg K_i \varphi$ and similarly for $\psi$ to get $(M,s)\ \vDash K_i\psi$, which is equivalent to $(M,s)\ \vDash\ K_i\varphi \rightarrow K_i\psi$ from propositional logic.
    
    To prove (4) and (5), we use the properties of the perspective function $\f_i$. The second property shows, a perspective function for agent $i$ on state $s$ converges after the first nested iteration, which means $(M,\f_i(s)) \equiv (M,\f_i(\f_i(s)))$. Therefore, whenever $(M, \f_i(s)) \vDash \psi$, then $\psi$ also holds in $(M,\f_i(f_i(s)))$, implying that $K_i\psi$ holds too. This is the case when $\psi$ is $K_i\varphi$ or $\neg K_i\varphi$, so (4) and (5) hold.
\end{proof}

\subsection{Group Knowledge}

From the basic visibility and knowledge definitions, in this section, we define group operators, including distributed and common visibility/knowledge.

\subsubsection{Syntax}

We extend the syntax of our language with group operators:
\[
 \varphi ::= \psi \mid \neg \varphi_1 \mid \varphi_2 \land \varphi \mid ES_G \alpha \mid EK_G \varphi \mid DS_G \alpha \mid DK_G \varphi \mid CS_G \alpha \mid CK_G \varphi,
 \]
in which $G$ is a set of agents and $\alpha$ is a variable $v$ or formula $\varphi$.

 Group formula $ES_G \alpha$  is read as: everyone in group $G$ sees variable/formula $\alpha$, and $EK_G\varphi$ represents that everyone in group $G$ knows $\varphi$. $DK_G$ is the distributed knowledge operator, equivalent to $D_G$ in Section~\ref{sec:background:epistemic-logic}, while $DS_G$ is its visibility counterpart: someone in group $G$ sees. Finally, $CK_G$ is common knowledge and $CS_G$ common visibility: ``it is commonly seen''.

\subsubsection{Semantics}


Let $G$ be a set of agents, $\varphi$ a formula, and $\alpha$ either a formula or a variable, then we can define the semantics of these group formula as follows:
\begin{itemize}
    \item $(M,s) \vDash ES_G\ \alpha$ iff $\forall i\in G$, $(M,s) \vDash S_i\ \alpha$
    \item $(M,s) \vDash EK_G\ \varphi$ iff $\forall i\in G$, $(M,s) \vDash K_i\ \varphi$
    \item $(M,s) \vDash DS_G\ \alpha$ iff $(M,s') \vDash \alpha$, where $s'=\bigcup\limits_{i\in G}\f_i(s)$
    \item $(M,s) \vDash DK_G\ \varphi$ iff $(M,s) \vDash \varphi$ and $(M,s) \vDash DS_G\ \varphi$
     \item $(M,s) \vDash CS_G\ \alpha$ iff $(M,s') \vDash \alpha$, where $s'=\f \cc(G,s)$
    \item $(M,s) \vDash CK_G\ \varphi$ iff $(M,s) \vDash \varphi$ and $(M,s) \vDash CS_G\ \varphi$,
\end{itemize}
\noindent 
in which $\f\cc(G,s)$ is state reached by applying composite function $\bigcap \limits_{i \in G} \f_i$ until it reaches its fixed point. That is, the fixed point $s'$ such that $\f\cc(G,s')=\f \cc(G,\bigcap \limits_{i \in G} f_i(s'))$.

Reasoning about common knowledge and visibility is more complex than other modalities. Common knowledge among a group is not only everyone in the group shares this knowledge, but also everyone knows other knows this knowledge, and so on, \emph{ad infinitum}. The infinite nature of this definition leads to some definitions that are untractable in some models. 

However, due to our restriction on the definition of states as variable assignments and our use of perspective functions, common knowledge is much simpler. Our intuition is based on the fact that each time we apply composite perspective function $\bigcap \limits_{i \in G} \f_i(s)$, the resulting state is either a proper subset of $s$  (smaller) or is $s$. By this intuition, we can limit common formula in finite steps.

The fixed point is a recursive definition. However, the following theorem shows that this fixed point always exists (even if it is empty), and the number of iterations is bound by the size of $|s|$, the state to which it is applied.


\begin{theorem}
    Function $\f \cc(G,s)$ converges on a fixed point $s'=\f\cc(G,s')$ within $|s|$ iterations. 
\end{theorem}
\begin{proof}
First, we prove convergence is finite; that is, when $s'=\f\cc(G,s')$, further `iterations' will result in $s'$; that is, $s'=\f\cc(G,\f\cc(G, s'))$.    Let $s_i=\f\cc(G,s_i)$, where $i$ is the number of iterations. Then, we have $s_{i+1}=\f\cc(G,s_i)=s_i$. Since $s_{i+1}=s_i$, we have $s_{i+2}=\f\cc(G,s_i)$, which means $s_{i+2}=s_i$. Via induction, we have that for all $k \geq i$, $s_k = s_i$. Therefore, once we reach convergence, it remains.

Next, we prove convergence within $|s|$ iterations.
By the intuition and definition of the perspective functions, $\f_k(s)\subseteq s$, and $s_{i+1}=(\bigcap \limits_{k \in G} f_k(s_i))$, we have $s_{i+1}\subseteq s_i$. Then, as we prove for the first point, if $s_i=s_{i+1}$, then we have reached a fixed point and no further iterations are necessary. Therefore, the worst case is when $s_{i+1} \subsetneq s_i$ and $|s_i| - |s_{i+1}| = 1$. There are most $|s|$ such worst case iterations until $s_i$  converges on an empty set. Therefore, the maximum  number of iterations is $|s|$.
\end{proof}

For each of the iteration, there are $|G|$ local states in group $G$ that need to be applied in the generalised union calculation, which can be done in polynomial time, and there there are at most $|s|$ steps. So, a poly-time algorithm for function $\f\cc$ exists.

\subsection{A brief note on expressiveness}


The intuitive idea about perspective functions is based on what agents can see, as determined by the current state. The relation between $t=\f_i(s)$ and $s$ corresponds roughly to accessibly relations $(s,t) \in \mathcal{K}_i$ in Kripke semantics. However, only focusing on what agent exactly knows/sees means overlooking those variables that agents are uncertain about. Perspective functions return one partial world that the agent $i$ is certain about, rather than a set of worlds that $i$ considers possible. The advantage is that applying a perspective function provides us only one state, rather than multiple states in Kripke semantics, preventing explosions in model size.

However, the reduced complexity loses information on the ``unsure" variables. Theoretically, $t=\f_i(s)$ from our model is a set intersection from all the $t$ that $(s,t)\in \mathcal{K}_i$. This eliminates disjunctive knowledge about variables; the only uncertainty being that an agent does not see a variable. For example, in the well-known \textit{Muddy Child} problem, the knowledge is not only generated by what each child can see by the others' appearance, which is modelled straightforwardly using perspective functions, but also can be derived from the statements made by their father and the response by other children. From their perspective, they would know exactly $m$ children are dirty, which can be handled by our model, as they are certain about it. While by the $k$-th times the father asked and no one responds, they can use induction and get the knowledge that at least $k$ children are dirty. By considering that there are two  possible worlds, where the number of dirty children is $m$ or $m+1$, Kripke structures keep both possible worlds until $m+1$ steps. Our model cannot keep this.

Therefore, although our model can handle preconditions and goals with disjunction, such as $K_i\ (v = e_1) \lor (v = e_2)]$, it cannot \emph{store} such disjunction in its `knowledge base'. Rather, agent $i$ knows $v$'s value is $e_1$ or its knows it is $e_2$. 

Despite this, possible worlds can be \emph{represented} using our model, using the same approach taken by \citet{DBLP:conf/aips/KominisG15} of representing Kripke structures in the planning language. We can then define perspective functions that implement Kripke semantics.

To summarise Section~\ref{sec:model}, we have defined a new model of epistemic logic, including group operators, in which states in the model are constrained to be just variable assignments. The complexity of common knowledge in this logic is bound by the size of the state and number of agents. In the following section, we show how to implement this in any F-STRIPS planner that supports external functions.

\section{Implementation}
\label{sec:implementation}

To validate our model and test its capabilities, we encode it within a planner and solve on some well-known epistemic planning benchmarks. Two key aspects in planning problem solving is the planning language and solver. Since, as mentioned in section \ref{sec:background:classical-planning}, we use BFWS($f_5$)\cite{DBLP:conf/ijcai/FrancesRLG17}, the advantage of F-STRIPS with external functions allows us to decompose the planning task from the epistemic logic reasoning. In this section, we encode our model into the F-STRIPS planning language and explain our use  of external functions for supporting this.

\subsection{F-STRIPS encoding}
Any classical F-STRIPS \cite{DBLP:conf/ijcai/FrancesRLG17} problem can be represented by a tuple $(V,D,O,I,G,\mathbb{F})$, where $V$ and $D$ are variables and domains, $O$, $I$, $G$ are operators, initial state, and goal states. $\mathbb{F}$ stands for the external functions, which allow the planner to handle problems that cannot be defined as a classical planning task, such as, the effects are too complex to be modelled by propositional fluents, or even the actions and effects has some unrevealed corresponding relations. In our implementation, external functions are implemented in C++, allowing any C++ function to be executed during planning.

In our epistemic planning model $(Agt,V,D,O,I,G)$ defined in Section~\ref{sec:model}, $Agt$ is a set of agent identifiers, and $V$ and $D$ are exactly same as in F-STRIPS. Operators $O$, initial states $I$ and  goal states $G$ differ to F-STRIPS only in that they contain epistemic formulae in preconditions, conditions, and goals.

\subsubsection{Epistemic Formula in Planning Actions}

There are two major ways to include epistemic formula in planning, using the formula as preconditions and conditions (on conditional effects) in operators, or using as  epistemic  goals.

Setting desirable epistemic formulae as goals is straightforward. For example, in Figure~\ref{fig:big_brother_logic}, if we want agent $a_1$ to know $a_2$ sees $b_1$, we could simply set the goal be $K_{a_1} S_{a_2} b_1$. However, there are some other scenarios that cannot be modeled by epistemic goals, including temporal constraints such as ``agent $a_1$ sees $b_2$ all the time" or ``target $b_4$ needs to secretly move to the other side without been seeing by any agents". Each of those temporal epistemic constraints can be handled by  a Boolean variable and getting the external functions to determine whether the constraint holds at each state of the plan. If the external function returns false, we `fail' that plan. 



In addition, epistemic formula can be in the precondition of the operators directly. For example, in Figure~\ref{fig:big_brother_logic}, if the scenario is continued surveillance of $b_2$ over the entire plan, then the operator {\turn}($a_1,d$) would have that either $S_{a_1} b_2 \lor S_{a_2} b_2$ as the precondition, after $a_1$ turns $d$ degree as one of the preconditions. 

As for the encoding $O$, $I$ and $G$, besides the classical F-STRIPS planning parts, all the formulae are encoded as calls to external functions.

\subsection{External Functions}\label{subsec:external-functions}

External functions take variables as input, and return the evaluated result based on the current search state. It is the key aspect that allows us to divide epistemic reasoning from search. Therefore, unlike compilation approaches to epistemic planning that compile new formula or normal forms that may never be required, our epistemic reasoning using lazy evaluation, which we hypothesise can significantly reduce the time complexity for most epistemic planning problems.

\subsubsection{Agent perspective functions}


As briefly mentioned in Section~\ref{subsubsec:knowledge-model}, the perspective function, $f_i : S \rightarrow S$, is a function that takes a state and returns the local state from the perspective of agent $i$. Compared to Kripke structures, the intuition is to only define which variables an agent sees. Individual and group knowledge all derive from this. 

Once we have domain-specific perspective functions for each agent, or just one implementation for homogenous agents, our framework implementation takes care of the remaining parts of epistemic logic.
Given a domain-specific perspective function, we have implemented a library of external functions that implement the semantics of $K_i$, $DS_i$, $DK_i$, $CS_i$, and $CK_i$, using the underlying domain-specific perspective functions.  The modeller simply needs to provide the perspective function for their domain, if a suitable one is not already present in our library.
 

For example, in Figure~\ref{fig:big_brother_logic}, global state $s$ covers the whole flat field. $\f_{a_1}(s)$ is the blue area, and $\f_{a_2}(s)$ is the yellow area, which means in agent $a_1$'s perspective, the world is the blue part, and for agent $a_2$, the world is only the yellow part. Further more, in agent $a_1$'s perspective, what agent $a_2$ sees, can be represented by the intersection between those two coloured areas, which is actually $\f_{a_2}(f_{a_1}(s))$. The interpretation is that, agent $a_1$ only considers state $l=\f_{a_1}(s)$ is the ``global state" for him, and inside that state, agent $a_2$'s perspective is $\f_{a_2}(l)$.

To be more specific about perspective functions, assume the global state $s$ contains all variables for \{$a_1, a_2, b_1, b_2,$ $ b_3, b_4$\}, such as locations, the directions agents are facing, etc. Based on the current set up from  Figure~\ref{fig:big_brother_logic}, we can implement $f_i$ for any agent $i$ with the following Euclidean geometric calculation:
\begin{equation}
     [|\arctan (\frac{|y(a_1)-y(a_2)|}{x(a_1)-x(a_2)}) - dir(a_1)| \leq \frac{range(a_1)}{2}] \bigoplus [|\arctan (\frac{|y(a_1)-y(a_2)|}{x(a_1)-x(a_2)}) - dir(a_1)| \geq \frac{360-range(a_1)}{2}]
\end{equation}





The perspective function takes all agents' locations, directions and vision angles, along with target location as input, and returns those agents and objects who fall instead these regions. On each variable for both agents, $\f_{a_1}(s)=\{a_1,a_2,b_2,b_3\}$ and $\f_{a_2}(s)=\{a_1,a_2,b_1,b_2\}$. Our library can then reason about knowledge operator. 

While such an approach could be directly encoded using proposition in  classical planning, we assert that the resulting encoding would be tediously difficult and error prone.

Their distributed knowledge would be derived from the union over their perspectives, which is the whole world $s$ except all variables for $b_4$. Both agents  would know every variables in the intersection of their perspectives, \{$a_1,a_2,b_2$\}, while it is also the same as what they commonly knows. 

However, if we alter the scenario a bit by turning $a_1$  $180^\circ$, then the $EK$ would work similar as one level perspective functions, while $CK$ is empty. In the new scenario, $\f_{a_1}(s)$ would become \{$a_1,b_1$\}, and $EK$ becomes \{$a_1,b_1$\}. When finding the converged perspective for $a_1$ and $a_2$, $\f_{a_1}(\{a_1,b_1\})$ stays the same, but the perspective for $a_2$ results in an empty world. The intersection over their perspective does not have $a_2$, which means $a_2$ does not exist in at least one of the agent's ($a_1$) perspectives from the group.

However, the above implementation of the perspective function is only useful for Euclidean planes. One of the novelties in this paper is that we define our implementation to support arbitrary perspective functions, which can be provided by the modeller as new external functions. Therefore, an agent's perspective function would be able to handle any problem with a proper set of visibility rules to be defined. Basically, in our implementation, a perspective function can take any state variable from the domain model, and converts it into its the agent's perspective state. Then, following the property that $\f_i(s)\subseteq s$, the function applies the domain-specific visibility rules on all the variables to gain the agent's local state. From there, the semantics outlined in Section~\ref{sec:model} are handled by our library.

Therefore, we can simply define our external function on different epistemic problem by providing different implementations of $f_i$. For example, the only difference between the external function's implementation on Corridor and Grapevine benchmarks from Section~\ref{sec:experiments}, is that in Corridor, the visibiility rules are that an agent can see within the current room and the adjacent room, while in Grapevine, the agent can only see within the current room. Moreover, in order to implement epistemic planning problems for e.g.\ a thtee-dimensional Euclidean plane, we need to modify the perspective functions based on the geometric model.

From a practical perspective, this means that modellers provide: (1) a planning model that uses epistemic formula; and (2) domain-specific implementations for $\f_1,\ldots,\f_n$. Linking the epistemic formula in the planning model to the perspective functions is delegated to our library.

\subsubsection{Domain dependent functions}
Domain dependent functions are customised relations for each set of problems, corresponding to $R(v_1,\dots,v_k)$ in Section~\ref{sec:model}, and can be  any domain specific functions that are implementable as external functions.

Overall, the implementation of our model is not as direct as other classical F-STRIPS planning problems or epistemic planning problems. However, as we demonstrate in the next section, it is scalable and flexible enough to find valid solutions to many epistemic planning problems.

\section{Experiments \& Results} \label{sec:experiments}
In this section, we evaluate our approach on four problems: \textit{Corridor} \cite{DBLP:conf/aips/KominisG15}, \textit{Grapevine} \cite{DBLP:conf/aaai/MuiseBFMMPS15}, \textit{Big Brother Logic ({BBL})} \cite{DBLP:conf/atal/GasquetGS14}, and \textit{Social-media Network} (\textit{SN}). Corridor and Grapevine are well-known epistemic planning problems, which we use to compare actual performance of our model against an existing state-of-the-art planner. BBL is a model of the Big Brother Logic in a two-dimensional continuous domain, which we used to demonstrate the expressiveness of our model. In addition, to demonstrate our model's capability of reasoning about group knowledge, we inherit the classical Gossip Problem, and create our own version, called Social-media Network. Hereafter, we assume any knowledge formula $K_a v$ is supplied with correct value $e$, which means its equivalent with $K_a v=e$, unless the value is specified. 


\subsection{Benchmark problems}

In this section, we briefly describe the corridor and grapevine problems, which are benchmark problems that we use to compare against \citeauthor{DBLP:conf/aaai/MuiseBFMMPS15}'s epistemic planner \cite{DBLP:conf/aaai/MuiseBFMMPS15}, which compiles to classical planning. 

\paragraph{Corridor}
The corridor problem was originally presented by \citeauthor{DBLP:conf/aips/KominisG15}\cite{DBLP:conf/aips/KominisG15}. It is about selective communication among agents. The basic set up is in a corridor of rooms, in which there are several agents. An agent is able to move around adjacent rooms, sense the secret in a room, and share the secret. The rule of communication is that when an agent shares the secret, all the agents in the same room or adjacent rooms would know.

The goals in this domain are to have some agents knowing the secret and other agents not knowing the secret. Thus, the main agent needs to get to the right room and communication to avoid the secret being overheard.

\paragraph{Grapevine}
Grapevine, proposed by \citeauthor{DBLP:conf/aaai/MuiseBFMMPS15}\cite{DBLP:conf/aaai/MuiseBFMMPS15}, is a similar problem to Corridor. With only two rooms available for agents, the scenario makes sharing secrets while hiding from others more difficult. The basic set up is each agent has their own secret, and they can share their secret among everyone in the same room. Since there are only two rooms, the secret is only shared within the room. The basic actions for agents are moving between rooms and sharing his secret. 

To evaluate computational performance of our model, we compare to \citeauthor{DBLP:conf/aaai/MuiseBFMMPS15}\cite{DBLP:conf/aaai/MuiseBFMMPS15}'s $PDKB$ planner. They have several test cases for Corridor and Grapevine. In addition, to test how the performance is influenced by the problem, we created new problems that varied some of the parameters, such as the number of agents, the number of goal conditions and also depth of epistemic relations. 

The $PDKB$ planner converts epistemic planning problems into classical planning, which results in a significant number of propositions when the depth or number of agents increased. We tried to submit the converted classical planning problem to the same planner that our model used, $BFWS(R)$ planner, to maintain a fair comparison. However, since the computational cost of the novelty check in $BFWS(R)$ planner increases with size of propositions, the planning costs was prohibitively expensive. Therefore, for comparison, we use the default \textit{ff} planner that is used by \citeauthor{DBLP:conf/aaai/MuiseBFMMPS15}.

We ran the problems  on both planners using a Linux machine with 1 CPU and 10 gigabyte memory. We measure the number of atoms (fluents) and number of  nodes generated during the search to compare the size of same problem modelled by different methods. We also measured the total time for both planners to solve the problem, and the time they taken to reasoning about the epistemic relations, which correspond to  the time taken to call external functions for our solution (during planning), and the time taken to convert the epistemic planning problem into classical planning problem in the $PDKB$ solution (before planning).




\begin{table}[h]
    \centering
    \begin{tabular}{ccccccccccccc}
         \toprule
        
         \multirow{3}{*}{Problem}
          & \multicolumn{3}{c}{Parameters} & \multicolumn{5}{c}{Our Model} & \multicolumn{4}{c}{PDKB} \\ 
           \cmidrule(lr){2-4} \cmidrule(lr){5-9} \cmidrule(lr){10-13}
          
          & \multirow{2}{*}{$|a|$} & \multirow{2}{*}{$d$} & \multirow{2}{*}{$|g|$} & \multirow{2}{*}{$|atom|$} & \multirow{2}{*}{$|gen|$} & \multirow{2}{*}{$|calls|$} & \multicolumn{2}{c}{TIME(s)} & \multirow{2}{*}{$|atom|$} & \multirow{2}{*}{$|gen|$}&\multicolumn{2}{c}{TIME(s)}
          \\ & & & & & & & {Calls} & Total &&& Compile & Total \\
         \midrule
         Corridor 
            & $3$ & $1$ & $2$ & $13$ & $15$ & $48$ & $0.001$ & $0.004$ & $54$ & $21$ & $0.148$& $0.180$\\
            & $7$ & $1$ & $2$ & $13$ & $15$ & $48$ & $0.002$ &  $0.005$ & $70$ & $21$ & $0.186$ & $0.195$\\
            & $3$ & $3$ & $2$ & $13$ & $15$ & $48$ & $0.003$ &  $0.007$& $558$ & $21$ & $0.635$ & $0.693$\\
            & $6$ & $3$ & $2$ & $13$ & $15$ & $48$ & $0.005$ &  $0.008$& $3810$ & $21$ & $5.732$ & $6.324$\\
            & $7$ & $3$ & $2$ & $13$ & $15$ & $48$ & $0.005$ &  $0.008$& $5950$ & $21$ & $9.990$ & $11.13$\\
            & $8$ & $3$ & $2$ & $13$ & $15$ & $48$ & $0.006$ &  $0.009$& $8778$ & $21$ & $14.14$ & $15.68$\\
            & $3$ & $4$ & $2$ & $13$ & $15$ & $48$ & $0.006$ &  $0.009$& $3150$ & $21$ & $3.354$ & $3.752$\\
            & $3$ & $5$ & $2$ & $13$ & $15$ & $48$ & $0.006$ &  $0.009$& $18702$ & $21$ & $25.69$ & $29.54$\\[2mm]
                  
         Grapevine 
            & $4$ & $1$ & $2$ & $346$ & $10$ & $48$ & $0.002$ & $0.006$& $96$ & $22$ & $0.429$ & $0.469$ \\
            & $4$ & $2$ & $2$ & $346$ & $10$ & $48$ & $0.002$ & $0.007$& $608$ & $5$ & $2.845$ & $3.168$ \\
            & $4$ & $1$ & $4$ & $346$ & $23$ & $144$ & $0.005$ & $0.009$& $96$ & $11$ & $0.428$ & $0.468$ \\
            & $4$ & $2$ & $4$ & $346$ & $23$ & $144$ & $0.006$ & $0.010$& $608$ & $11$ & $2.885$ & $3.178$ \\
            & $4$ & $1$ & $8$ & $346$ & $368$ & $1200$ & $0.040$ & $0.047$ & $96$ & $529$ & $0.381$ & $0.455$ \\
            & $4$ & $2$ & $8$ & $346$ & $368$ & $1200$ & $0.050$ & $0.057$ & $608$ & $1234$ & $3.450$ & $4.409$ \\
            & $4$ & $3$ & $8$ & $346$ & $368$ & $1200$ & $0.068$ & $0.073$ & $4704$ & $14$ & $28.66$ & $30.72$ \\
            & $8$ & $1$ & $2$ & $546$ & $18$ & $48$ & $0.003$ & $0.010$ & $312$ & $5$ & $3.025$ & $3.321$ \\
            & $8$ & $2$ & $2$ & $546$ & $18$ & $48$ & $0.003$  & $0.010$& $4408$ & $5$ & $54.35$ & $58.80$ \\
            & $8$ & $1$ & $4$ & $546$ & $43$ & $144$ & $0.008$  & $0.016$ & $312$ & $11$ & $2.546$ & $2.840$ \\
            & $8$ & $2$ & $4$ & $546$ & $43$ & $144$ & $0.008$  & $0.016$ & $4408$ & $11$ & $55.33$ & $59.78$ \\
            & $8$ & $1$ & $8$ & $546$ & $1854$ & $4528$ & $0.238$  & $0.268$& $312$ & $2002$ & $2.519$ & $3.752$\\
            & $8$ & $2$ & $8$ & $546$ & $1854$ & $4528$ & $0.322$  & $0.294$ & $4408$ & $4371$ & $54.90$ & $228.1$\\
            & $8$ & $3$ & $8$ & $546$ & $1854$ & $4528$ & $0.394$  & $0.421$ & $-$ & $-$ & $-$ & $-$\\


        \bottomrule
    \end{tabular}
    \caption{Results for the Corridor and Grapevine Problems}
    \label{tab:computational}
\end{table}
We show the results of the problems in Table~\ref{tab:computational}, in which $|a|$  specifies the number of agents, $d$ the maximum depth of a nested epistemic query, $|g|$ the number of goals, $|atom|$ the number of atomic fluents, $|gen|$ the number of generated nodes in the search, and $|calls|$ the number of calls made to external functions.

From the results, it is clear that the complexity of the $PDKB$ approach grows exponentially on both number of the agents and depth of epistemic relations (we ran out of memory in the final Grapevine problem), while in our approach,  those features do not have a large affect. However, epistemic reasoning in our approach (calls to the external solver), has a significant influence on the performance for our solution. Since the F-STRIPS planner we use checks each goal at each node in the search, the complexity is in $O(|g|*|gen|)$. While this is exponential in the size of the original problem (because $|gen|$ is exponential), the compuational cost is significantly lower than the compilation in the $PDKB$ approach.

\subsection{Big Brother Logic}

Big Brother Logic (BBL) is a problem that first discussed by \citeauthor{DBLP:conf/atal/GasquetGS14} \cite{DBLP:conf/atal/GasquetGS14}. The basic environment is on a two-dimensional space called ``Flatland'' without any obstacles. There are several stationary and transparent cameras; that is,  cameras can only rotate, and do not have volume, so they do not block others' vision. In our scenario, we allow cameras to also move in Flatland.

\subsubsection{Examples}
Let $a_1$ and $a_2$ be two cameras in Flatland. Camera $a_1$ is located at $(5,5)$, and camera $a_2$ at $(15,15)$. Both cameras have an $90^{\circ}$ range. Camera $a_1$ is facing north-east, while camera $a_2$ is facing south-west. There are three objects with values $o_1=1$, $o_2=2$ and $o_3=3$, located at $(1,1)$, $(10,10)$ and $(19,19)$ respectively. For simplicity, we assume only camera $a_1$ can move or turn freely, and camera $a_2$, $o_1$, $o_2$ and $o_3$ are fixed. Figure \ref{fig:big_brother_exp} visualises the problem set up.

\begin{figure}[!ht]
    \centering
    \newcommand{\scale}{0.4}
\newcommand{\range}{30}
\newcommand{\size}{3}
\newcommand{\oversize}{6}
\makeatletter
\newcommand{\gettikzxy}[3]{%
  \tikz@scan@one@point\pgfutil@firstofone#1\relax
  \edef#2{\the\pgf@x}%
  \edef#3{\the\pgf@y}%
}
\makeatother

\begin{tikzpicture}
    \coordinate (top_left) at (0,20*\scale);
    \coordinate (bottom_left) at (0,0);
    \coordinate (top_right) at (20*\scale,20*\scale);
    \coordinate (bottom_right) at (20*\scale,0);

    \coordinate (origin) at (0,0);
    \coordinate (a1) at (5*\scale,5*\scale);
    \coordinate (a1_up) at (5*\scale,20*\scale);
    \coordinate (a1_down) at (20*\scale,5*\scale);
    \coordinate (a1_dir) at (6*\scale,6*\scale);
    \coordinate (a2) at (15*\scale,15*\scale);
    \coordinate (a2_up) at (0,15*\scale);
    \coordinate (a2_down) at (15*\scale,0);
    \coordinate (a2_dir) at (14*\scale,14*\scale);
        
    \begin{scope}[transparency group]
        \begin{scope}[blend mode=multiply]
            \fill[ opacity=0.5,blue!30] (a1) -- (a1_up) -- (top_right) -- (a1_down) -- cycle;
            \fill[ opacity=0.5,yellow!30] (a2) -- (a2_up) -- (bottom_left) -- (a2_down) -- cycle;
        \end{scope}
    \end{scope}
    
    \draw[thick, black,->] (a1) -- (a1_dir);
    \draw[thick, black,->] (a2) -- (a2_dir);

    \draw[thick] let    \p{1} = (a1)    in (a1 |- origin)
        node[circle,fill=red!80,inner sep=2pt,
        label={[align=center]below:
                $a_1$ \\ 
                (5,5)
                }] at (\x1,\y1) {};
    \draw[->, black] (a1) -- (a1_up);
    \draw[->, black] (a1) -- (a1_down);

    \draw[thick] let    \p{1} = (a2)    in (a1 |- origin)
        node[circle,fill=red!80,inner sep=2pt,
        label={[align=center, xshift=0.5cm]below:
                $a_2$ \\ 
                (15,15)
                }] at (\x1,\y1) {};
    \draw[->, black] (a2) -- (a2_up);
    \draw[->, black] (a2) -- (a2_down);


    \draw[black] (top_left) -- (bottom_left) -- (bottom_right) -- (top_right) -- cycle;

    \coordinate (b1) at (10*\scale,10*\scale);
    \coordinate (b2) at (1*\scale,1*\scale);
    \coordinate (b3) at (19*\scale,19*\scale);
    
    \draw[thick] let    \p{1} = (b1)    in (a1 |- origin)
        node[circle,fill=black!80,inner sep=1pt,
        label={[align=center]below:
                $o_2$ \\ 
                (10,10)
                }] at (\x1,\y1) {};
                
    \draw[thick] let    \p{1} = (b2)    in (a1 |- origin)
        node[circle,fill=black!80,inner sep=1pt,
        label={[align=center]above:
                $o_1$ \\ 
                (1,1)
                }] at (\x1,\y1) {};    

    \draw[thick] let    \p{1} = (b3)    in (a1 |- origin)
        node[circle,fill=black!80,inner sep=1pt,
        label={[align=center,xshift=-0.3cm, yshift=-1.0cm]：
                $o_3$ \\ 
                (19,19)
                }] at (\x1,\y1) {};    
                
\end{tikzpicture}
    \caption{Example for Big Brother Logic set up}
    \label{fig:big_brother_exp}
\end{figure}
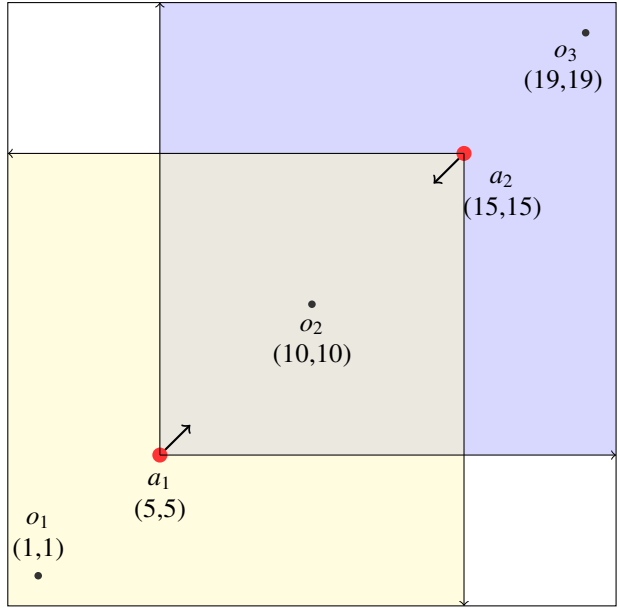

This problem can be represented by the tuple $(V,D,O,I,G,\mathbb{F})$, where:

\begin{itemize}
    \item $V\ =\ \{x,\ y,\ direction,\ query\}$
    \item $D\ :\ dom(x)=dom(y)=\{-20,\dots,20\}$; $dom(direction)=\{-179,\dots,180\}$; and, $dom(query)=\{0,1\}$
    \item $O\ :\ $ {\move($dx,dy$)} and {\turn($d$)}
    \item $I\ =\ \{x=5,\ y=5,\ direction=45\}$
    \item $G\ =\ \{query=1\}$
    \item $\extf\ :\ $ {\checking}: $query$ $\mapsto \{true,false\}$,
\end{itemize}
\noindent in which $query$ is a goal query, which we describe later.

This is just a simple example to demonstrate our model. Variables $x$ and $y$ represent coordinates of camera $a_1$, and $direction$ determines which way $a_1$ is facing. Since $a_2$ and all other objects are fixed, we model them in an external state handled by the external functions, which lightens the domain and reduces the state space. However, we could also model the positions of these as part of the planning model if desired. For the domain of the variables, although the F-STRIPS planner we used does not support using real numbers,  using integers is enough to show that our model can work on continuous problems.

We need to check the knowledge queries in the actions (precondition, conditional-effects), or goals. Both action \move($dx,dy$) and action \turn($d$) can change agents' perspectives, and therefore, can influence knowledge. To simplify the problem, instead of moving $1$ or turning $1$ degree per action, we can set up a reasonable boundary for actions, such as turning $\pm45$ degree per action and move at speed of $2$ for each direction on $x$ and $y$.


As for the goal conditions, some queries can be achieved for the problem in Figure~\ref{fig:big_brother_exp} without doing any actions, such as:
\begin{enumerate}
    \item Single Knowledge query: $\neg K_{a_2} o_3$; $K_{a_1} o_3$
    \item Nested Knowledge query: $\neg K_{a_1} S_{a_2} o_3$; $S_{a_1} S_{a_2} o_3$
    \item Group Knowledge query: $K_{a_1,a_2} o_3$; $K_{a_1,a_2} o_2$
    \item Distributed Knowledge query: $\neg DK_{a_1,a_2} (o_1,3)$; $DK_{a_1,a_2} (o_1,1)$
    \item Common Knowledge query: $CK_{a_1,a_2} o_2$; $CK_{a_1,a_2} S_{a_1} o_3$
\end{enumerate}
Most of them are intuitive. From goal 2, although $S_{a_1} S_{a_2} o_3$ is true because $a_1$ can see $a_2$'s location, range of vision and direction, so $a_1$ knows whether $a_2$ can see $o_3$, the formula $K_{a_1} S_{a_2} o_3$ is false because there is no action that $a_1$ can do to make $a_2$ see $o_3$. For goal 5, $CK_{a,b} S_{a_1} o_3$ is true without any action, because, by calling the external function, the common local state for $a_1$ and $a_2$ would be the location of all three values, both $a_1$ and $a_2$ and the value of $o_2$. Then, $S_{a_1} o_3$ would be evaluated true based on the common local state.

In addition, there are some query that would be achieved through valid plans:
\begin{enumerate}
    \item $K_{a_1,a_2} o_1$: \move($-2,-2$), \move($-2,-2$)
    \item $CK_{a_1,a_2} o_1$: \move($-2,-2$), \move($-2,-2$)
    \item $S_{a_2} S_{a_1} o_1:$ \move($-2,2$), \move($-2,2$)
    \item $K_{a_1} o_1 \land \neg K_{a_2} K_{a_1} o_1$(BBL11): {\move}($-2,1$), {\move}($-2,2$), {\move}($-1,2$), {\move}($0,2$), {\move}($0,2$), {\move}($0,2$), {\turn}($-45$), {\turn}($-44$)
    \item $\neg K_{a_1} S_{a_2} S_{a_1} o_1 \land S_{a_1} o_1$(BBL12): {\turn}($-44$), {\turn}($-45$), {\turn}($-45$), {\move}($1,2$), {\move}($2,2$), {\move}($2,2$), \\{\move}($2,2$), {\move}($2,2$), {\move}($2,2$), {\move}($-1,2$)
\end{enumerate}
The first one is clear. There is more than one way to let both of them know value $o_1$, and the planner returns the optimal solution. The second one is also intuitive: to achieve common knowledge in a BBL problem, they need to both see the item and both see each other. The difference between the last two are a bit trickier. To avoid $a_2$ that knows whether $a_1$ can see $o_1$, the cheapest plan returned by planner was for $a_1$ to move out of $a_2$'s eye sight. The last one is the most difficult to solve. Not only should $a_1$ see $o_1$, but also $a_1$ should know that originally $a_2$ cannot see that $a_1$ sees $o_1$. This is done by decomposing into three facts: ``$a_1$ sees $o_1$";``$a_2$ cannot see whether $a_1$ sees $o_1$"; and, ``$a_1$ can see that whether $a_2$ can see whether $a_1$ sees $o_1$".

\subsubsection{Results}

Table~\ref{tab:sn} shows the results for our problems in the BBL domain. A plan of $\infty$ means that the problem is unsolvable -- no plan exists. While the perspective function in BBL depends on a \textit{geometric model} based on agent's position, direction and facing angle, the results show that  with proper usage of our functional STRIPS planner, we can represent continuous domains. Our agent epistemic solver is able to reason about other the agents' epistemic states (vision) and derive plans based on these for non-trivial intricate goal that we believe would be difficult to encode propositionally, demonstrating that our model can handle important problems in vision-based domains.

\noindent
\begin{table}[!th]
    \centering
    
        \begin{tabular}{lcccccccccl}
         \toprule

        \multirow{3}{*}{Problem}
        & \multicolumn{4}{c}{Parameters} & \multicolumn{5}{c}{Performance} &  \\
         \cmidrule(lr){2-5} \cmidrule(lr){6-10}
        & \multirow{2}{*}{$|a|$} & \multirow{2}{*}{$d$} & \multirow{2}{*}{$|g|$} & \multirow{2}{*}{$|p|$} & \multirow{2}{*}{$|gen|$} & \multirow{2}{*}{$|exp|$} & \multirow{2}{*}{$|calls|$} & \multicolumn{2}{c}{TIME(s)} &\multirow{2}{*}{Goal} \\
        & & & & & & & & {$calls$} & Total & \\
        \midrule
        BBL01 & $2$ & $1$ & $1$ & $0$ & $1$ & $0$ & $2$ & $0.000$ & $0.002$ &
        $K_{a_1} o_2$ \\
        BBL02 & $2$ & $1$ & $1$ & $2$ & $115$ & $2$ & $232$ & $0.007$  & $0.009$& $K_{a_1} o_1$ \\
        BBL03 & $2$ & $1$ & $1$ & $\infty$ & $605160$ & $all$ & $1210320$ & $36.1$ & $78.2$& $K_{a_2} o_3$ \\
        BBL04 & $2$ & $2$ & $1$ & $2$ & $115$ & $2$ & $232$ & $0.015$ & $0.017$ & $K_{a_1} K_{a_2} o_1$ \\
        BBL05 & $2$ & $1$ & $1$ & $0$ & $1$ & $0$ & $2$ & $0.000$ & $0.002$ & $DK_{a_1,a_2}\{o_1,o_2,o_3\}$ \\
        BBL06 & $2$ & $1$ & $1$ & $0$ & $1$ & $0$ & $2$ & $0.000$ & $0.002$ & $EK_{a_1,a_2} o_2$ \\
        BBL07 & $2$ & $1$ & $1$ & $2$ & $115$ & $2$ & $232$ & $0.019$ & $0.021$& $EK_{a_1,a_2} \{o_1,o_2\}$ \\
        BBL08 & $2$ & $1$ & $1$ & $0$ & $1$ & $0$ & $2$ & $0.000$ & $0.002$ & $CK_{a_1,a_2} o_2$ \\
        BBL09 & $2$ & $1$ & $1$ & $2$ & $115$ & $2$ & $232$ & $0.048$ & $0.050$& $CK_{a_1,a_2} \{o_1,o_2\}$ \\
        BBL10 & $2$ & $2$ & $1$ & $2$ & $115$ & $2$ & $232$ & $0.016$ & $0.018$& $K_{a_1} DK_{a,b} \{o_1,o_2,o_3\}$ \\
        BBL11 & $2$ & $2$ & $2$ & $8$ & $59260$ & $7509$ & $187332$ & $7.650$ & $8.382$ & $K_{a_1} o_1 \land \neg K_{a_2} K_{a_1} o_1$ \\
        BBL12 & $2$ & $3$ & $2$ & $9$ & $15842$ & $592$ & $47626$ & $2.380$ & $2.472$ & $S_{a_1} o_1 \land  \neg K_{a_1} S_{a_2} S_{a_1} o_1 $ \\
        \bottomrule
    \end{tabular}
    \caption{Experiments Results for the Big Brother Logic domain}
    \label{tab:bbl}
\end{table}

\subsection{Social-media Network}

The \emph{Social-media Network} (SN) domain is an abstract network which agents can befriend each other to read their updates, etc., based on typical social media models. We extend two-way one-time communication channels from a classical gossip problem \cite{DBLP:conf/ecai/CooperHMMR16} into two-way, all-time communication channels, and add the concepts of secret messages. By decomposing secrets into messages and posting through an agent's friendship network, we model how  secrets can be shared between a group of individuals not directly connected without anyone on the network knowing the secret, and some secrets can be shared within a group excepting for some individuals. The former would be ,ike spies sharing information with each other through the resistance's personal page, and the latter would be a group arranging a surprise party for a mutual friend.

\subsubsection{Examples}
Let $a,b,c,d,e$ be five agents in the SN, with friendship links shown in Figure~\ref{fig:social_media_network}. Their friend relations are represented by full lines between each agent. 

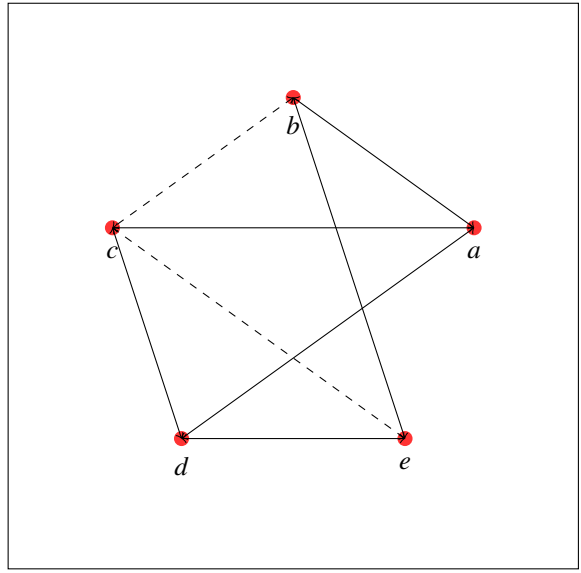
\begin{figure}
    \centering
\newcommand{\base}{2.5}
\newcommand{\piii}{180}
\newcommand{\range}{30}
\newcommand{\size}{3}
\newcommand{\oversize}{6}
\makeatletter
\newcommand{\gettikzxy}[3]{%
  \tikz@scan@one@point\pgfutil@firstofone#1\relax
  \edef#2{\the\pgf@x}%
  \edef#3{\the\pgf@y}%
}
\makeatother

\begin{tikzpicture}

    \coordinate (a) at ({\base*cos(1/10*\piii)}, {\base*sin(1/10*\piii)});
    
    \coordinate (b) at (0, {\base});
    
    \coordinate (c) at ({-\base*cos(1/10 * \piii)}, {\base*sin(1/10*\piii)});
    
    \coordinate (d) at ({-\base*sin(1/5*\piii)}, {-\base*cos(1/5*\piii)});
    
    \coordinate (e) at ({\base*sin(1/5* \piii)}, {-\base*cos(1/5 * \piii)});
    
    \draw[black] ({1.5*\base},{1.5*\base}) -- (-{1.5*\base},{1.5*\base}) -- (-{1.5*\base},-{1.5*\base}) -- ({1.5*\base},-{1.5*\base}) -- cycle;
    \draw[thick] 
        node[circle,fill=red!80,inner sep=2pt,
        label={[align=center]below:$a$}] at (a) {};
    \draw[thick] 
        node[circle,fill=red!80,inner sep=2pt,
        label={[align=center]below:$b$}] at (b) {};
    \draw[thick] 
        node[circle,fill=red!80,inner sep=2pt,
        label={[align=center]below:$c$}] at (c) {};
    \draw[thick] 
        node[circle,fill=red!80,inner sep=2pt,
        label={[align=center]below:$d$}] at (d) {};
    \draw[thick] 
        node[circle,fill=red!80,inner sep=2pt,
        label={[align=center]below:$e$}] at (e) {};
    
    \coordinate (origin) at (0,0);
    \coordinate (a1) at (0,0);
    \coordinate (a1up) at ({cos(\range)*\oversize},{sin(\range)*\oversize});
    \coordinate (a1down) at ({cos(\range)*\oversize},-{sin(\range)*\oversize});
    \coordinate (a2) at ({cos(\range)*\size*2},0);
    \coordinate (a2up) at ({cos(\range)*\size*2-cos(\range)*\oversize},{sin(\range)*\oversize});
    \coordinate (a2down) at ({cos(\range)*\size*2-cos(\range)*\oversize},-{sin(\range)*\oversize});
    
    \coordinate (b1) at (-{cos(\range)*\size},0);
    \coordinate (b2) at ({cos(\range)*\size},0);
    \coordinate (b3) at ({cos(\range)*\size*3},0);
    \coordinate (b4) at ({cos(\range)*\size},{sin(\range)*\oversize*3/4});
    
    \draw[black,<->] (a) -- (b);
    \draw[black,<->] (a) -- (c);
    \draw[black,<->] (a) -- (d);
    \draw[black,<->] (b) -- (e);
    \draw[black,<->] (c) -- (d);
    \draw[black,<->] (d) -- (e);
    \draw[black,<->,dashed] (b) -- (c);
    \draw[black,<->,dashed] (e) -- (c);

\end{tikzpicture}



    \caption{Example for Social-media Network}
    \label{fig:social_media_network}
\end{figure}

Let $g$ be a friend for all agents and $g$ wants to share a secret. We assume the social network is in $g$'s perspective directly, and the network is fixed for simplicity. Let all the epistemic queries that we concern of be a set $Q$, and $p_1,p_2,p_3$ as three parts of the secret $P$. Any problem by this set up can be represented by a tuple $(A,V,D,O,I,G,\mathbb{F})$, where:
\begin{itemize}
    \item $A$ = \{$a,b,c,d,e$\} 
    \item $V$ = $\{(\mathit{friended}\ i\ j)$, $(post\ p)$ $(q)\mid i,j\in A,\ p\in P,\ q\in Q\}$
    \item $D$ : $dom(\mathit{friended}\ i\ j)$ = $dom(q)$ = \{$0,1$\}, $dom(post\ p)=A$, where $i,j\in A,\ p\in P,\ q\in Q$ 
    \item $O$ : {\post}($i,p$), where $i\in A,\ p \in P$
    \item $I$ = \{ $(\mathit{friended}\ a\ b)=1$, $(\mathit{friended}\ a\ c)=1$, $(\mathit{friended}\ a\ d)=1$, $(\mathit{friended}\ b\ e)=1$, $(\mathit{friended}\ c\ d)=1$, ${(\mathit{friended}\ d\ e)=1}$ \}
    \item $G$: see below
    \item $\extf$ : ({\checking}: $q$) $\mapsto \{true,false\}$
\end{itemize}

$(\mathit{friended}\ i\ j)$ represents whether $i$ and $j$ are friends with each other, which is a domain dependent relation, and $q$ covers all the epistemic queries. $(\post\ i\ p)$ specifies that the message $p$ is posted on agent $i$'s page. The $I$ covers the friendship relations in Figure~\ref{fig:social_media_network}, and no message has been posted yet. Similarly, the action \post\ is the only source for epistemic relation changes. Therefore, we update all desired queries $Q$ as condition effects.

Goals that we have tested are shown in Table~\ref{tab:sn}.
For some epistemic formulae between $a$ and $b$, since they are friends, simply posting the message in any of their personal page is sufficient to establish common knowledge about the information in that post. But for the knowledge between $a$ and $e$, for example, $EK_{a,e} p_1$, the message needs to be posted on the page of a mutual friend, such as agent $b$. In addition, since $a$ and $e$ are not friends, in each of their perspectives of the world, there is no information (variables) describing others. Therefore, both $EK_{a,e} EK_{a,e} p_1$ and $CK_{a,e} p_1$ are not possible without changing the network structure. 

Below are some of the sample goals that we have tested, with the identifying name from Table~\ref{tab:sn} and the plan that achieves that goal:

\begin{itemize}
    \item $K_a p_1$ (SN01), $K_a K_b p_1$ (SN02), $EK_{a,b} p_1$ (SN03) and $CK_{a,b} p_1$ (SN06): {\post}($a$,$p_1$)
    \item $CK_{a,e} p_1$ (SN07): $\infty$
    \item $K_a \{p_1,p_2,p_3\}$ (SN08), $EK_{a,b} \{p_1,p_2,p_3\}$ (SN04) and $DK_{a,b} \{p_1,p_2,p_3\}$ (SN05):\\ {\post}($a$,$p_1$), {\post}($a$,$p_2$), {\post}($a$,$p_3$)
\end{itemize}

A plan of $\infty$ means that the problem is unsolvable.

The other types of goals are secretive:
\begin{itemize}
    \item $K_a \{p_1,p_2,p_3\} \land \neg K_b \{p_1,p_2,p_3\}$ (SN09): {\post}($a$,$p_1$), {\post}($a$,$p_2$), {\post}($c$,$p_3$)
    \item $K_a \{p_1,p_2,p_3\} \land \neg K_b \{p_1,p_2,p_3\}\land \neg K_c \{p_1,p_2,p_3\}$ (SN10): {\post}($a$,$p_1$), {\post}($b$,$p_2$), {\post}($c$,$p_3$)
\end{itemize}

The aims are to share the whole secret with $a$ while $b$ must not know the whole secret, but can know some of it. Some parts of it, such as $p_3$, needs to be shared in the page that $b$ does not have access to. Then, $c$ must also not know the secret, the secret now needs to be posted in the way that $b$ and $c$ do not see some parts respectively, while $a$ sees all the parts.

Finally, we look on those two desired scenarios in the introduction of SN:
\begin{itemize}
    \item Sharing with a spy:
    \\$K_a \{p_1,p_2,p_3\} \land \neg K_b \{p_1,p_2,p_3\} \land \neg K_c \{p_1,p_2,p_3\}  \land \neg K_d \{p_1,p_2,p_3\}  \land \neg K_e \{p_1,p_2,p_3\}$ (SN11): \\{\post}($a$,$p_1$), {\post}($b$,$p_2$), {\post}($c$,$p_3$)
    \item Surprise party:
    \\$\neg K_a \{p_1,p_2,p_3\} \land K_b \{p_1,p_2,p_3\} \land K_c \{p_1,p_2,p_3\}  \land K_d \{p_1,p_2,p_3\}  \land K_e \{p_1,p_2,p_3\}$ (SN13): $\infty$
\end{itemize}

Sharing a secret to some specific individual without anyone else knowing the secret can be done with the current network. However, if we alter the problem a bit by adding a friend relation between $b$ and $c$ (SN12), and apply the same goal conditions as SN11, no plan would be found by the planner, because $c$ sees everything $a$ can see, and there is no way to share some information to $a$ without $c$ seeing it.

For sharing a secret surprise party for agent $a$ among all the agents without $a$ knowing it, the messages needs to shared in such a way that $a$ is not able to get a complete picture of the secrets. In the current set up of the problem (SN13), since $a$ sees everything seen by $c$, there is no way to held a surprise party without $a$ knowing it. However, by adding a friend relation between $e$ and $c$ (SN14), the planner returns with the plan: {\post}($e$,$p_1$), {\post}($e$,$p_2$), {\post}($e$,$p_3$).

Table~\ref{tab:sn} shows the results for our problems in the social-media network domain. The results show that is able to reason about nested knowledge and also group knowledge to achieve an intricate goal, which means our model can handle variety of knowledge relations at same time within reasonable time complexity.

\noindent
\begin{table}[!th]
    \centering
    
        \begin{tabular}{lcccccccccl}
         \toprule

        \multirow{3}{*}{Problem}
        & \multicolumn{4}{c}{Parameters} & \multicolumn{5}{c}{Performance} &  \\
         \cmidrule(lr){2-5} \cmidrule(lr){6-10}
        & \multirow{2}{*}{$|a|$} & \multirow{2}{*}{$d$} & \multirow{2}{*}{$|g|$} & \multirow{2}{*}{$|p|$} & \multirow{2}{*}{$|gen|$} & \multirow{2}{*}{$|exp|$} & \multirow{2}{*}{$|calls|$} & \multicolumn{2}{c}{TIME(s)} &\multirow{2}{*}{Goal} \\
        & & & & & & & & {$calls$} & Total & \\
        \midrule
        SN01 & $5$ & $1$ & $1$ & $1$ & $16$ & $2$ & $42$ & $0.002$ & $0.004$ & $K_i p_1$ \\
        SN02 & $5$ & $2$ & $1$ & $1$ & $16$ & $2$ & $42$ & $0.003$ & $0.005$ & $K_i K_j p_1$ \\
        SN03 & $5$ & $1$ & $1$ & $1$ & $16$ & $2$ & $42$ & $0.003$ &  $0.004$ & $EK_{i,j} p_1$ \\
        SN04 & $5$ & $1$ & $1$ & $3$ & $216$ & $92$ & $3286$ & $0.484$ & $0.489$ & $EK_{i,j} \{p_1,p_2,p_3\}$ \\
        SN05 & $5$ & $1$ & $1$ & $3$ & $216$ & $92$ & $3286$ & $0.566$ & $0.571$& $DK_{i,j} \{p_1,p_2,p_3\}$ \\
        SN06 & $5$ & $1$ & $1$ & $1$ & $16$ & $2$ & $42$ & $0.006$ & $0.007$ & $CK_{i,j} p_1$ \\
        SN07 & $5$ & $1$ & $1$ & $\infty$ & $216$ & $all$ & $7776$ & $0.825$ &  $0.838$ & $CK_{i,k} p_1$ \\
        SN08 & $5$ & $1$ & $1$ & $3$ & $216$ & $92$ & $3286$ & $0.216$ & $0.221$ & $K_i \{p_1,p_2,p_3\}$ \\
        SN09 & $5$ & $1$ & $2$ & $3$ & $306$ & $107$ & $7652$ & $0.759$ &  $0.767$ & $K_i \{p_1,p_2,p_3\} \land \neg K_j \{p_1,p_2,p_3\}$ \\
        SN10 & $5$ & $1$ & $3$ & $3$ & $614$ & $189$ & $20334$ & $2.049$ &  $2.069$ & $K_i \{p_1,p_2,p_3\} \land \neg K_{j\land k} \{p_1,p_2,p_3\}$ \\
        SN11 & $5$ & $1$ & $5$ & $3$ & $901$ & $265$ & $47570$ & $4.841$ &  $4.886$ & $K_i \{p_1,p_2,p_3\} \land \neg K_{other} \{p_1,p_2,p_3\}$ \\
        SN12$^{1}$ & $5$ & $1$ & $5$ & $\infty$ & $2808$ & $all$ & $505400$ & $57.5$ &  $58.0$ & $K_i \{p_1,p_2,p_3\} \land \neg K_{other} \{p_1,p_2,p_3\}$ \\
        SN13 & $5$ & $1$ & $2$ & $\infty$ & $432$ & $all$ & $31104$ & $5.589$ &  $5.629$ & $ \neg K_i \{p_1,p_2,p_3\} \land K_{other} \{p_1,p_2,p_3\}$ \\
        SN14$^{2}$ & $5$ & $1$ & $2$ & $3$ & $418$ & $278$ & $19964$ & $4.049$ &  $4.073$ & $ \neg K_i \{p_1,p_2,p_3\} \land K_{other} \{p_1,p_2,p_3\}$ \\
        \bottomrule
    \end{tabular}
    \caption{Experiments Results for the Social-media Network domain}
    \label{tab:sn}
\end{table}

\subsection{Discussion}


Overall, computationally, our solution  outperforms $PDKB$, which is state-of-the-art for epistemic planning problems. The number of agents and depth of epistemic relations do not not increase the computation time as rapidly as the $PDKB$ planner. In the terms of expressiveness, our solution demonstrates its capability to handle variety of complex epistemic relations, such as, nested knowledge, distributed knowledge and common knowledge.

The results show that the computational time depends heavily on how many times have the external functions are called, which is actually determined by the number of generated nodes and expanded nodes. Moreover, the number of nodes involved in the search is impacted by some  factors, such as, the length of the plan, the algorithm that the planner uses, and also the scale of the problem itself.

The results also show that the external solver takes up a large part of the execution time. This is a prototype implementation and this represents an opportunity for performance optimisation of our code base.


\section{Conclusions}

In this work, we introduced a new epistemic planning model called the  agent perspective model, driven from the intuition: ``What you know is what you see". This perspective based model allows us to evaluate epistemic formula, even nested, distributed or common epistemic relations, based on the simple concept of defining an agent's local state. Then, by separating the planning task from epistemic reasoning with functional STRIPS, we proposed an expressive and flexible solution for most of the epistemic planning problems without an expensive pre-compilation step. We implemented our model well-known epistemic planning benchmarks and two new scenarios based on different perspective functions. The results not only show our model can solve the epistemic benchmarks efficiently, but also demonstrate a variety types of epistemic relation can be handled. Our work is the first to delegate epistemic reasoning to an external solver.

For future work, there are three ways to extend our model. First, extending the model to belief rather than knowledge. The success of our model is dependent on the property $f_i(s) \subseteq s$ for perspective functions, which implies beliefs cannot be false. Extending to belief would be an important step. Second, we can improve our model by allowing simplified disjunctive knowledge relations, such as that proposed by \citet{DBLP:conf/aaai/MillerFMPS16}. Finally, investigating event-based epistemic planning domains within our model would be another potential research direction. Our intuition is that those problems would challenge our model, however, it would be possible to including seeing events as well as seeing variables or agents.

\subsection*{Acknowledgements}

The authors thank Andreas Herzig for his insightful discussions on the link between knowledge and seeing, and for inspiring the idea of the social network domain.




\bibliographystyle{model1-num-names}







\end{document}